\documentclass[11pt]{article}
\usepackage[abbrvbib, preprint]{jmlr2e}

\usepackage{enumitem}
\usepackage[dvipsnames]{xcolor}

\usepackage{amsfonts,amsmath,amssymb,dsfont,bbm}
\usepackage{algorithm}
\usepackage[noend]{algpseudocode}

\usepackage[mathscr]{euscript}
\usepackage{caption}
\usepackage{subcaption}
\usepackage{cancel}

\usepackage{tikz}
\usetikzlibrary{arrows.meta}
\usetikzlibrary{shapes.geometric}

\usepackage{enumitem}
\usepackage{multirow}
\usepackage{makecell}

\let\jmlrproof\proof \let\endjmlrproof\endproof
\let\proof\relax \let\endproof\relax

\usepackage{amsthm}
\newif\ifusinglinuxtexlive
\usinglinuxtexlivetrue
\ifusinglinuxtexlive
    \usepackage{thmtools}
\fi

\let\proof\jmlrproof \let\endproof\endjmlrproof

\usepackage[bibliography=common,appendix=append]{apxproof} 
\usepackage[nameinlink,capitalize]{cleveref}

\crefname{equation}{}{} 
\numberwithin{equation}{section}

\newtheorem{defin}{Definition}
\newtheoremrep{lem}[defin]{Lemma}
\newtheoremrep{thm}[defin]{Theorem}
\newtheoremrep{coro}[defin]{Corollary}
\newtheorem{prop}[defin]{Proposition}

\crefname{defin}{Definition}{Definitions}
\crefname{lem}{Lemma}{Lemmas}
\crefname{thm}{Theorem}{Theorems}
\crefname{coro}{Corollary}{Corollaries}
\crefname{prop}{Proposition}{Propositions}

\newcommand{\doo}{\mathrm{do}}
\newcommand{\cond}{\,\vert\,}

\newcommand{\F}{{\mathbb{F}}}

\newcommand{\cF}{\mathcal{F}}

\newcommand{\cC}{\mathcal{C}}

\newcommand{\cD}{\mathcal{D}}

\newcommand{\cM}{\mathcal{M}}

\newcommand{\cG}{\mathcal{G}}

\newcommand{\cH}{\mathcal{H}}

\newcommand{\Rcv}{\mathrm{Re}}
\newcommand{\Emt}{\mathrm{Em}}
\newcommand{\Prnt}{\mathrm{Pa}}
\newcommand{\Chl}{\mathrm{Ch}}
\newcommand{\Anc}{\mathrm{An}}
\newcommand{\Dsc}{\mathrm{De}}

\newcommand{\Do}{\mathrm{do}}

\newcommand{\Pa}[2]{\mathrm{Pa}_{#1}^{*}(#2)}
\newcommand{\Ch}[2]{\mathrm{Ch}_{#1}^{*}(#2)}

\newcommand{\indep}{\perp\kern-7pt\perp}

\newcommand{\dashleftrightarrow}{\mathrel{
    \tikz[baseline]{\draw[dash pattern=on .25em off .2 em,<->](0,.58ex)--(1.75em,.58ex)}}}

\definecolor{Cblue}{rgb}{0.85, 0.95, 1.0}
\definecolor{Cyellow}{rgb}{1.0, 0.95, 0.65}
\definecolor{Cgreen}{rgb}{0.8, 1.0, 0.7}
\definecolor{Cmisc}{rgb}{0.9, 0.9, 0.9}
\definecolor{Cinvalid}{rgb}{1.0, 0.7, 0.65}
\definecolor{darkyellow}{rgb}{0.7, 0.65, 0.0}

\usepackage{booktabs}

\allowdisplaybreaks
\numberwithin{equation}{section}

\usepackage{lastpage}
\jmlrheading{}{2026}{1-\pageref{LastPage}}{10/26}{}{}{Jani Nykänen, Otto Tabell, Santtu Tikka and Juha Karvanen}

\ShortHeadings{Invariance of Clustering Operations in Causal Effect Identification}{Nykänen, Tabell, Tikka and Karvanen}

\begin{document}

\title{Invariance of Clustering Operations in Causal Effect Identification}

\author{\name Jani Nykänen \email jani.m.nykanen@jyu.fi \\
        \name Otto Tabell \email otto.h.e.tabell@jyu.fi \\
        \name Santtu Tikka \email santtu.tikka@jyu.fi \\
        \name Juha Karvanen \email juha.t.karvanen@jyu.fi \\
        \addr Department of Mathematics and Statistics \\
        FI-40014 University of Jyvaskyla, Finland}

\editor{}

\maketitle

\begin{abstract}
{\noindent
Clustering variables in causal graphs reduces the size of the graph and simplifies causal inference. However, arbitrary clustering can alter crucial causal relations among variables and lead to erroneous conclusions. While the identifiability of a causal effect in the clustered graph implies the identifiability in the original graph under mild conditions, nonidentifiability in clustered graph does not imply nonidentifiability in the original graph without further assumptions. When both identifiability and nonidentifiability are preserved, the clustering operation is called identification invariant. We present a broad class of clustering operations that are identification invariant based on conditions related to the c-components of the original graph. Finally, we demonstrate use of the results in practical settings.
} 

\smallskip


\end{abstract}


\section{Introduction} \label{sec_intro}

A causal model is a simplified representation of real causal mechanisms. It is unnecessary and often practically unfeasible to describe lower level causal mechanisms when the interest lies in a higher level. Clustering variables, such as genetic factors, job history or health status, and representing them with a single vertex in a directed acyclic graph (DAG) enables modeling high-level causal relationships without specifying the detailed structure within the cluster. 
Applying Occam's razor, a causal model should be ``as coarse as possible but not coarser''. 

In addition to high-level causal modeling, clustering of variables in causal graphs is beneficial in other ways. Clustering reduces the size of a graph, making it easier to comprehend the visual representation. Clustering may also speed up identification algorithms and simplify the identifying functionals \citep{thk2023}. Assumptions on clusters may provide useful information for causal discovery algorithms \citep[and references therein]{madaleno2026coarsening}. 

In the literature, clustering in graphical models has been approached from several perspectives and under diverse terminology. Clusters may be referred to as groups \citep{wahl2024discovery}, super-nodes \citep{tennant2021use}, partitions \citep{madaleno2026coarsening}, macro variables \citep{ferreira2025identifying} or vector variables \citep{ninad2025causal}. Clustering may be called coarsening \citep{madaleno2026coarsening}, summarization \citep{zeng2025causal} or (an instance of) causal abstraction \citep{rubenstein2017causal, beckers2019abstracting}. Graphs that contain clustered variables may be called cluster graphs or C-DAGs \citep{anand2023cdags}. Clustering is relevant for all areas of causal inference and has recently gained attention especially in causal discovery. \citet[Appendix A]{madaleno2026coarsening} summarizes the current state of the art in causal discovery. Existing approaches differ in whether the clusters are assumed to be known beforehand \citep[e.g.][]{wahl2024discovery,gobler2025nonlinear} or if they are learned jointly with the causal structure \citep[e.g.][]{massidda2024learning,madaleno2026coarsening}.

This paper focuses on the impact of clustering on the identifiability of causal effects. This topic has previously been studied in settings where (a) both the original and the clustered graph are acyclic \citep{thk2023,ttk2025,anand2023cdags}; (b) the original graph is acyclic but the clustered graph may contain cycles \citep{ferreira2025identifying,yvernes2026relaxing}; and (c) both graphs may be cyclic \citep{ferreira2026identifying}. In all settings, under mild assumptions, it has been shown that if a causal effect is identifiable in the clustered graph, then the corresponding effect is identifiable in the original graph. \citet{anand2023cdags,ferreira2025identifying,ferreira2026identifying,yvernes2026relaxing} consider \emph{macro} causal effects, that is, effects of entire clusters on other clusters, and establish the converse identifiability result only relative to the class of compatible graphs: nonidentifiability in the clustered graph implies nonidentifiability in \emph{some} graph that yields the clustered graph, but not necessarily in the graph at hand. 

\citet{thk2023} and \citet{ttk2025} use a special class of clusters called \emph{transit clusters} and consider effects not involving the clustered variables, and determine when nonidentifiability in the clustered graph implies nonidentifiability in the original graph, i.e., when \emph{identification invariance} holds. Such an approach has a clear advantage: we may conclusively evaluate the identifiability of given causal effects merely with the clustered graph. The downside is that the converse result requires stricter conditions related to the cluster: the mild conditions described above---more precisely, acyclicity of the clustered graph---are insufficient for establishing identification invariance. This is illustrated by \Cref{fig_violated_paths_no_identif,fig_plain_not_coherent}, where a causal effect is identifiable in the original graph but not in the clustered graph. To rule out such settings, transit clusters impose rather strict conditions on how the cluster connects to its parents and children.

In this paper, our aim is to continue in the line of research of identification invariance. Our starting point is the observation that identifiability always transfers from the clustered graph to the original graph, provided that the clustering operation does not create cycles. Establishing identification invariance therefore reduces to a single question: when does nonidentifiability in the clustered graph imply nonidentifiability in the original graph? As nonidentifiability from the joint observed distribution is equivalent to the presence of a graphical structure called a hedge \citep{shpitser_pearl_1,shpitser2023doesidalgorithmfail}, this amounts to finding conditions under which clustering cannot create a hedge that is not already present in the original graph. We present a class of clusters that we call \emph{coherent}, which generalizes the class of transit clusters, and we then present a class of identification invariant clustering operations for coherent clusters based on conditions related to c-components \citep{tian2002general, shpitser_pearl_1}. These results, in turn, generalize some of the earlier identifiability results by \citet{thk2023}, thus widening the applicability of clustering approaches related to identification invariance.

\smallskip

The article is organized as follows. In \cref{sec_notation}, we introduce some essential preliminary definitions and notation. We proceed to define clustering in \cref{sec_clustering} and consider properties of clustering operations. In \cref{sec_identif} we study the effect of clustering on identifiability of causal effects, focusing on determining universally identification invariant clustering operations. We demonstrate the use of identifiability results in \cref{sec_app}. Concluding thoughts are offered in \cref{sec_discussion}.


\section{Notation and Definitions} \label{sec_notation}

A directed acyclic graph (DAG) is a pair $\cG = (V, E)$, where $V$ is the set of vertices and $E$ is the set of directed edges, denoted by arrows $v_i \rightarrow v_j$, such that there are no directed cycles in $\cG$. Supposing sets $A, B \subseteq V$ in a DAG $\cG = (V, E)$, we use the following notation and conventions:
\begin{itemize}

	\item $\Prnt_{\cG}^*(A)$, $\Chl_{\cG}^*(A)$, $\Anc_{\cG}^*(A)$ and $\Dsc_{\cG}^*(A)$ denote the parents, children, ancestors and descendants of the vertices of $A$ in $\cG$, respectively. The subscript can be omitted if the underlying graph $\cG$ is clear from the context. If the set $A$ is included in the respective sets, then the asterisk ${}^*$ is omitted (e.g., $\Prnt_{\cG}(A) = \Pa{\cG}{A} \cup A$). If $A = \left\{ a \right\}$, then the curly brackets are omitted (e.g., $\Ch{\cG}{a} = \Ch{\cG}{\left\{ a \right\}}$).

    \item $\cG[A]$ denotes the induced subgraph $(A, F)$ of $\cG$, where $F$ contains those edges of $E$ that have both endpoints in $A$. $\cG[\overline{A}, \underline{B}]$ denotes an induced edge subgraph obtained from $\cG$ by removing incoming edges to $A$ and outgoing edges of $B$ (if no edges incoming to $A$ or outgoing of $B$ are removed, then the corresponding symbol is omitted, i.e., $\cG[\overline{A}]$ or $\cG[\underline{B}]$).

\end{itemize}

\smallskip

We study DAGs induced by structural causal models.
\begin{defin}[Structural causal model]
A structural causal model (SCM) is a tuple $\cM = (X_V, X_{\widetilde{U}}, \F, p)$ consisting of the following objects:
\begin{itemize}
	\item A vector of observed random variables $X_V$, indexed by the set $V$.
	\item A vector of unobserved random variables $X_{\widetilde{U}}$, indexed by the set $\widetilde{U}$.
	\item A family of functions $\F = (f_v)_{v \in V}$ such that each $f_v \in \F$ is a map from the state space of $X_{\widetilde{U}(v) \cup X_{\Prnt(v)}}$ to the state space of $X_v$, where $\widetilde{U}(v) \subseteq \widetilde{U}$ and $\Prnt(v) \subseteq V \setminus \{ v \}$, such that $\F$ forms a map from the state space of $X_{\widetilde{U}}$ to the state space of $X_V$. Accordingly, each random variable $v \in V$ is defined by the corresponding structural equation $X_v = f_v(X_{\Prnt(v)}, X_{\widetilde{U}(v)})$.
	\item A joint probability distribution $p$ of the unobserved variables $X_{\widetilde U}$.
\end{itemize}
\end{defin}

We only consider recursive SCMs, i.e., models where the unique solution of $\F$ in terms of $\widetilde{U}$ is obtained by recursively substituting each $X_w\in (X_w)_{w\in\Prnt(v)}$ in each function $f_v\in \F$ with their corresponding functional representation with respect to $f_w \in \F$. Unobserved vertices that are present in exactly one structural equation are called \emph{dedicated error terms}, and the set of dedicated error terms is denoted by $X_{\bar{U}}$. Correspondingly, we use $X_U$ to denote the set of unobserved variables that are present precisely in two structural equations. We focus solely on semi-Markovian causal models, where each unobserved variable can be present in at most two structural equations, that is, $X_{\widetilde{U}} = X_{U \cup \bar{U}}$. A motivation for this is described below. 

An SCM $\cM$ induces a causal graph, that is, a DAG $\cG = \cG(\cM) = (V \cup U, E)$, where the vertices correspond with the indexing sets of the random variables in $\cM$, and the directed edges in $E$ point from the members of $\Prnt(v)$ and $U(v)$ to $v$. The dedicated error terms are omitted from the causal graphs, i.e., all members of $U$ have exactly two children in $\cG$.

Vertices of $V$ and $U$ are called observed and unobserved vertices (of $\cG$), respectively.  Moreover, unobserved vertices with two children are presented in figures using bidirected (dashed) arcs without specifying the unobserved vertex itself. A bidirected path is a path that consists solely of bidirected arcs, that is, a path of the type $v_1 \leftarrow u_1 \rightarrow v_2 \leftarrow \cdots \leftarrow u_k \rightarrow v_k$, where $v_1, \ldots, v_k \in V$ and $u_1, \ldots, u_k \in U$. If one does not restrict oneself to semi-Markovian causal models, then the induced DAG can have unobserved vertices with more than two children, in which case one can apply latent projection \citep{verma93} to project the induced DAG into a DAG where each unobserved vertex has exactly two children. An alternative way to formulate such models is to use acyclic directed mixed graphs \citep[ADMGs,][]{richardson2003admg, richardson2023admg}, which explicitly allow bidirected arrows in graphs.

An intervention $\Do(x_B)$, where $B \subset V$, on a model $\cM$ forces the variables $X_B$ to take  specific values $x_B$. The intervention creates a submodel $\cM_{B}$, which is defined as follows:
\begin{defin}[Submodel]
A submodel $\cM_{B}$ of a structural causal model $\cM$ obtained from an intervention $\Do(x_B)$ is an SCM $(X_V, X_{\widetilde{U}}, \F', p)$, where the new set of functions $\F'$ is obtained from $\F$ by replacing the functions $(f_b)_{b \in B} \subseteq \F$ by constant functions that output the corresponding value in $x_B$. 
\end{defin}

The marginal distribution of $X_{V\setminus B}$ in $\cM_{B}$ is called the interventional distribution, and it is denoted by $p(\cdot \mid \Do(x_B))$. The causal effect of $X_B$ on $X_A$ is defined as $p(x_A \mid \Do(x_B))$. Identifiability is an important property of causal effects, and it is defined as follows:

\begin{defin}[Identifiability]
Given a causal graph $\cG = (V \cup U, E)$ and two disjoint sets $A, B \subseteq V$, the causal effect $p(x_A \mid \Do(x_B))$ is identifiable from $p(x_V)$ in $\cG$ if it is uniquely computable from $p(x_V)$ in every causal model that induces $\cG$.
\end{defin}

Since we are only interested in causal effects that are identifiable from the joint distribution, we henceforth simply say that a causal effect is identifiable in $\cG$ without specifying the distribution $p(x_V)$. Furthermore, we restrict ourselves to study identifiability from positive distributions, that is, distributions that satisfy $p(x_V) > 0$ for all value assignments $x_V$, which ensures that all causal effects and identifying functionals are well-defined.

Given a (non)identifiable causal effect, operations that modify causal graphs and preserve the (non)identifiability of the given causal effect (or a family of causal effects) are called identification invariant \citep{ttk2025}. A formal definition of this property is given below.
\begin{defin}[Identification invariance] \label{def_identif_invariance}
Suppose an operation $\Omega$ that maps a causal graph $\cG$ to a causal graph $\cG' = \Omega(\cG)$.  

Then:
\begin{itemize}
    \item $\Omega$ is \textit{identification invariant} for a causal effect $p(x_{A} \mid \Do(x_B))$ if the causal effect is either identifiable in $\cG$ and $\cG'$ or nonidentifiable in $\cG$ and $\cG'$.
    \item $\Omega$ is identification invariant for a family of causal effects $\cC = (p(x_{A_i} \mid \Do(x_{B_i}) ))_{i=1}^m$ if the operation $\Omega$ is identification invariant for each effect $p(x_{A_i} \mid \Do(x_{B_i})) \in \cC$.
\end{itemize}
\end{defin}

An important graphical structure associated with identifiability of causal effects is the c-component.
\begin{defin}[c-component]
A causal graph $\cG = (V \cup U, E)$ is a c-component if for every $v_i, v_j \in V$ there exists a bidirected path $v_i \dashleftrightarrow v_{i + 1} \dashleftrightarrow \cdots \dashleftrightarrow v_{j - 1} \dashleftrightarrow v_j$, where the notation $v_k \dashleftrightarrow v_\ell$ denotes a fork $v_k \leftarrow u \rightarrow v_\ell$ for some $u \in U$.
\end{defin}

In other words, a causal graph is a c-component if and only if a subset of the bidirected arcs forms a spanning tree over all observed vertices in the graph \citep{shpitser_pearl_1}. A c-tree is a special type of c-component used to characterize identifiability of direct effects.

\begin{defin}[c-tree]
A causal graph $\cG = (V \cup U, E)$ is a c-tree if $\cG$ is a c-component and every observed vertex $v \in V$ has at most one child. If there is a vertex $y \in V$ such that $\cG[\Anc_{\cG}(y)] = \cG$, then $\cG$ is $y$-rooted c-tree
\end{defin}

A c-forest generalizes the concept of a c-tree by allowing a root set instead of a single root vertex.

\begin{defin}[c-forest]
A causal graph $\cG = (V \cup U, E)$ is a c-forest with a root set $Y = \left\{ v \in V : \Dsc_{\cG}^*(v) = \emptyset \right\}$ if $\cG$ is a c-component and every observed vertex $v \in V$ has at most one child.
\end{defin}

Two c-forests with a special relation form a hedge, which characterize identifiability of general causal effects.

\begin{defin}[Hedge] \label{def_hedge}
Let $A, B \subseteq V$ be disjoint sets in a causal graph $\cG = (V \cup U, E)$. A pair $(\cF, \cF_{*})$ is a hedge for the causal effect $p(x_A \mid \Do(x_B))$ provided that:
\begin{enumerate}[label=\rm{(\arabic*)}]
    \item $\cF = (V_{\cF} \cup U_{\cF}, E_{\cF})$ and $\cF_{*} = (V_{\cF_{*}} \cup U_{\cF_{*}}, E_{\cF_{*}})$ are $R$-rooted c-forests, where $R \subset V$; \label{prop_hedge_forests}
    \item $V_{\cF} \cap B = \emptyset$, $V_{\cF_{*}} \cap B \neq \emptyset$ and $\cF \subseteq \cF_{*}$; and \label{prop_hedge_inclusion}
    \item $R \subseteq \Anc_{\cG[\overline{B}]}(A)$. \label{prop_hedge_root}
\end{enumerate}
\end{defin}

The identification algorithm \textsc{ID} by \citet{shpitser_pearl_1} is a sound and complete algorithm that verifies whether a causal effect is identifiable from $p(x_V)$ in a causal graph, and outputs its identifying formula if identifiability holds. The importance of hedges stems from the fact that \textsc{ID} algorithm fails only if it detects a hedge for the studied causal effect, that is, nonidentifiability is equivalent to the presence of a hedge.

\begin{prop}[{\citet[Proposition 1]{shpitser2023doesidalgorithmfail}}] \label{thm_hedge_criterion}
A causal effect $p(x_A \mid \Do(x_B))$ is identifiable in a causal graph $\cG = (V \cup U, E)$ if and only if there does not exist a hedge for $p(x_{A} \mid \Do(x_{B}))$ in $\cG$.
\end{prop}


\section{Clustering in Causal Graphs} \label{sec_clustering}

\begin{toappendix}
  \label{apx_clusteringproof}
\end{toappendix}

We are interested in reducing the size of a DAG by replacing a vertex set (i.e., a cluster) by a single vertex --- an operation which we call \emph{clustering}. In this section we give a rigorous definition of clustering. As an important class of clusters, we introduce coherent clusters and study their properties. The effect of clustering on identifiability of causal effects is further studied in \cref{sec_identif}.

Clustering of a vertex set is defined as follows:

\begin{defin}[Clustering] \label{def_clustering}
Let $\cG = (V \cup U, E)$ be a causal graph. Clustering a vertex set $T \subset V$ induces a causal graph $\cG' = (V' \cup U', E')$ such that:
\begin{itemize}
    \item $V' = (V \setminus T) \cup \left\{ t \right\}$, where the vertex $t$ represents the set $T$.
    \item $U' = U \setminus \left\{ u \in U \mid \Ch{\cG}{u} \subset T \right\}$, that is, all unobserved vertices whose both children are in $T$ are excluded from $U'$.
    \item $E'$ is obtained from $E$ by removing all edges connected to $T$ (including the edges $t_i \rightarrow t_j$, where $t_i, t_j \in T$), and adding new edges connecting $t$ to $V \setminus T$ in the following way:
    \begin{enumerate}[label=$*$]
        \item the edge $v_i \rightarrow t$ is in $E'$ if and only if there is a $t_i \in T$ and an edge $v_i \rightarrow t_i$ in $E$.
        \item the edge $t \rightarrow v_j$ is in $E'$ if and only if there is a $t_j \in T$ and an edge $t_j \rightarrow v_j$ in $E$.
    \end{enumerate}
\end{itemize}
\end{defin}
Given an SCM $\cM$ that induces a causal graph $\cG$, then there exists a causal model $\cM'$ which induces the clustered graph $\cG'$ such that $\cM'$ can be constructed from $\cM$ by replacing the set of variable $X_T = (X_{t_i})_{t_i \in T}$ with a single vector-valued variable $X_t$ that has members of $X_T$ as its components, and modifying the structural equations of $X_t$ and its children accordingly. Furthermore, $\cM'$ is necessarily semi-Markovian, because by the definition of clustering the unobserved vertices in the clustered graph $\cG'$ cannot have more than two children.

We note that if there is a $u \in U$ whose both children are in the set $T$, then after clustering the vertex $u$ has only one child, namely the vertex $t$, and by convention we omit unobserved vertices with only one child (i.e, the dedicated error terms) in causal graphs, which is the reason why such unobserved vertices are not included in $U'$.

Moreover, we disallow clustering of unobserved vertices (i.e., $T \cap U = \emptyset$) for the following reasons. First, if one clusters an unobserved vertex whose other child is not in the clustered set, then it is not always clear whether the vertex representing the cluster should be considered an observed or an unobserved vertex. Second, if one only clusters those unobserved vertices whose both children are in the cluster, then they become dedicated errors terms in the clustered model, so they are implicitly clustered together with their children in the clustering operation. Third, this way we do not need to specify the unobserved variables in the graphs, and we can denote them by bidirected arcs to increase visual clarity.

\smallskip

In the case of clustering, the operation $\Omega$ appearing in the definition of identification invariance (\cref{def_identif_invariance}) is the clustering operation $\Omega_{T}(\cG) = \cG'$, where $\cG$ is the original graph and $\cG'$ is the graph induced by clustering of the given vertex set $T$. We are especially interested in those clustering operations that are identification invariant for all causal effects not involving variables of the clustered set, a property which we call universal identification invariance. Formally we define this property as follows:

\begin{defin} [Universal identification invariance]
Suppose a vertex set $T \subset V$ in a causal graph $\cG = (V \cup U, E)$. The clustering of $T$ is \emph{universally identification invariant} if the operation is identification invariant for all causal effects $p(x_A \mid \Do(x_B))$ such that $A \cap B = \emptyset$ and $A, B \subset V \setminus T$.
\end{defin}

\smallskip

We introduce the following notation for clusters that are augmented with their unobserved parents whose both children are inside the cluster. This notation is useful when considering the parents of the cluster, since often we want to include only those unobserved parents whose other child is outside the cluster.

\begin{defin}
Suppose a vertex set $T \subset V$ in a causal graph $\cG = (V \cup U, E)$. We define 
\begin{equation*}
T^+ = T \cup \left\{ u \in U \mid \Ch{\cG}{u} \subset T \right\}.
\end{equation*}
\end{defin}

We proceed to define clustering equivalent vertex sets.

\begin{defin}[Clustering equivalent vertex sets]
Suppose a set of vertices $T \subset V$ in a causal graph $\cG = (V \cup U, E)$, and let $\cG' = (V' \cup U', E')$ be the graph induced by the clustering of $T$ with the vertex $t$ representing the set $T$. Vertex sets $W \subset V \cup U$ and $W' \subset V' \cup U'$ are called clustering equivalent provided that $W \setminus T = W' \setminus \left\{ t \right\}$ and $T \subset W$ if and only if $t \in W'$.
\end{defin}

In this paper we define clustering as a graph operation that outputs a causal graph, that is, a DAG. However, without further restrictions, \cref{def_clustering} allows clustering operations that induce directed cycles, an example of which is provided in \Cref{fig_cycles}. Furthermore, arbitrary clustering can introduce new directed connections not present in the original graph, with a possible side effect that a causal effect that is identifiable in the original graph may not be identifiable in the clustered graph. Examples of such scenarios are provided in \Cref{fig_violated_paths_no_identif,fig_plain_not_coherent}.

\begin{figure}
\centering
\subcaptionbox{
Clustering the set $T_1 = \left\{ t_1, t_2 \right\}$ creates a directed cycle in the clustered graph.\label{fig_cycles}
}[0.25\textwidth]{
\begin{tikzpicture}[scale=0.8, transform shape, >={Stealth[scale=1]}]
\node[shape=circle,draw=black] (v1) at (0,0) {$v_1$};
\node[shape=circle,draw=black, fill=Cinvalid] (r) at (0,-2) {$t_1$};
\node[shape=circle,draw=black, fill=Cinvalid] (e) at (0,-4) {$t_2$};
\node[shape=circle,draw=black] (v2) at (-2,-2) {$v_2$};
\node[shape=circle,draw=black] (v3) at (-2,-4) {$v_3$};
\node[shape=circle,draw=black] (v4) at (0,-6) {$v_4$};

\path [->] (v1) edge node[right] {} (r);
\path [->] (v2) edge node[right] {} (v1);
\path [->] (v4) edge node[right] {} (v3);
\path [->] (v3) edge node[right] {} (v2);
\path [->] (e) edge node[right] {} (v4);

\end{tikzpicture}
}
\hfill
\subcaptionbox{Clustering the vertex sets $T_2 = \left\{ t_1, t_2 \right\}$ forms a new directed path $b \rightarrow \cdots \rightarrow a$, which also renders the causal effect $p(x_a \mid \Do(x_b))$ nonidentifiable.\label{fig_violated_paths_no_identif}}[0.275\textwidth]{

\begin{tikzpicture}[scale=0.8, transform shape, >={Stealth[scale=1]}]
\node[shape=circle,  minimum size=0.80cm, draw=black] (b) at (-2,2) {$b$};
\node[shape=circle,draw=black] (v1) at (-4,2) {$v_1$};
\node[shape=circle,  minimum size=0.80cm, draw=black] (a) at (-2,-4) {$a$};
\node[shape=circle,draw=black, fill=Cinvalid] (t1) at (-4,0) {$t_1$};
\node[shape=circle,draw=black, fill=Cinvalid] (t2) at (-2,-2) {$t_2$};
\node[shape=circle,draw=black] (v2) at (-2,0) {$v_2$};
\node[shape=circle,draw=black] (v3) at (-4,-2) {$v_3$};

\path [->] (b) edge node[right] {} (v1);
\path [->] (v1) edge node[right] {} (t1);
\path [->] (t1) edge node[right] {} (v2);
\path [->] (t2) edge node[right] {} (v2);
\path [->] (t2) edge node[right] {} (a);
\path [->] (v3) edge node[right] {} (t1);
\path [->] (v3) edge node[right] {} (t2);

\path [<->,dashed] (v1) edge[bend left=50] node[right] {} (b);
\path [<->,dashed] (b) edge[bend left=50] node[right] {} (a);

\end{tikzpicture}
}
\hfill
\subcaptionbox{$T_3 = \left\{ t_1, t_2 \right\}$ is not coherent, because in every path from $v_1$ to $a$ the non-endpoint vertices are not contained in $T_3$. Consequently identifiability of $p(x_a \mid \Do(x_b))$ is lost in clustering of $T_3$. \label{fig_plain_not_coherent}}[0.375\textwidth]{

\begin{tikzpicture}[scale=0.8, transform shape, >={Stealth[scale=1]}, every node/.style={minimum height=0.80cm}]
\node[shape=circle,draw=black] (v1) at (0,0) {$v_1$};
\node[shape=circle,draw=black] (v2) at (2,1) {$v_2$};
\node[shape=circle,draw=black] (v3) at (4,3) {$v_3$};
\node[shape=circle,draw=black, fill=Cinvalid] (t1) at (2,-1) {$t_1$};
\node[shape=circle,draw=black, fill=Cinvalid] (t2) at (4,-1) {$t_2$};
\node[shape=circle,draw=black] (b) at (6,1) {$b$};
\node[shape=circle,draw=black] (a) at (6,-2) {$a$};

\path [->] (v1) edge node[right] {} (v2);
\path [->] (v1) edge node[right] {} (t1);
\path [->] (v2) edge node[right] {} (b);
\path [->] (v3) edge node[right] {} (b);
\path [->] (b) edge node[right] {} (a);
\path [->] (t2) edge node[right] {} (a);
\path [->] (t1) edge node[right] {} (b);

\path [<->,dashed] (v1) edge[bend left=50] node[right] {} (v3);
\path [<->,dashed] (v3) edge[bend left=50] node[right] {} (b);
\path [<->,dashed] (a) edge[bend left=50] node[right] {} (v1);

\end{tikzpicture}
}
\caption{Examples of clustering operations with undesired side effects.}
\label{fig_not_coherent}
\end{figure}
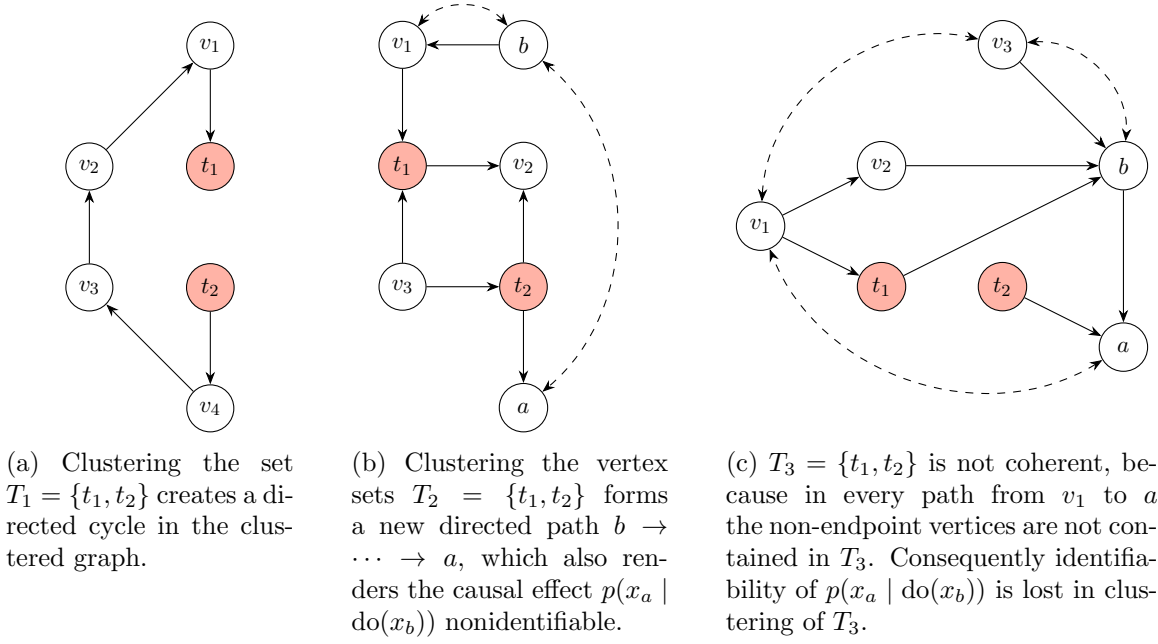

To remedy the abovementioned cases, we introduce a \textit{coherent} vertex set for which we ensure that there exists a directed path from every parent of the set to every child of the set such that the path goes directly through the set. Formally, we define coherence as follows.

\begin{defin}[Coherence] \label{def_coherence}
A set of vertices $T \subset V$ in a causal graph $\cG = (V \cup U, E)$ is coherent provided that for all $p \in \Pa{\cG}{T^+}$ and $c \in \Ch{\cG}{T}$ there exists a directed path $p \rightarrow t_1 \rightarrow \cdots \rightarrow t_m \rightarrow c$ such that $t_1, \ldots, t_m \in T$ and $m \geq 1$.
\end{defin}

We say that a clustering operation is coherent if the clustering is applied to a coherent set of vertices in $\cG$. In the following lemma we show that coherent clustering does not create directed cycles or new directed connections not present in the original graph. Additionally, we show that a directed path can intersect a coherent cluster at most once, that is, there is at most one maximal subpath that consists solely of vertices of the cluster. The proof of the lemma is presented in Appendix~\ref{apx_clusteringproof}.

\begin{lemrep} \label{lemma_coherence_properties}
Let $T \subset V$ be a coherent cluster in a causal graph $\cG = (V \cup U, E)$, and let $\cG'$ be the graph obtained by clustering the set $T$ with the vertex $t$ representing the set $T$. Then:
\begin{enumerate}[label=\rm{(\roman*)}]
    \item Clustering of $T$ does not create directed cycles, that is, $\cG'$ is a DAG. \label{prop_coherence_no_cycles} 
    \item If a directed path $w_1 \rightarrow \cdots \rightarrow w_k$, where $w_1 \in V \cup U$ and $w_2, \ldots, w_k \in V$, contains members of $T$, then the vertices $\left\{ w_1, \ldots, w_k \right\} \cap T$ form a subpath $t_1 \rightarrow \cdots \rightarrow t_m$ of $w_1 \rightarrow \ldots \rightarrow w_k$. In other words, each directed path can intersect a coherent vertex set at most once. \label{prop_coherence_one_maximal_subpath}
    \item For all $w \in V \setminus T$ it holds $\Anc_{\cG}(w) \setminus T^+ = \Anc_{\cG'}(w) \setminus \left\{ t \right\}$. Consequently $w_i, w_j \in (V \cup U) \setminus T^+$ are connected via a directed path in $\cG$ if and only if they are connected via a directed path in $\cG'$. \label{prop_coherence_ancestors} 
\end{enumerate}
\end{lemrep}

\begin{proof}
\textbf{Case \ref{prop_coherence_no_cycles}:} Let us assume a directed cycle $v_i \rightarrow \cdots \rightarrow v_i$ in $\cG'$. Without loss of generality we can assume that $v_i \in V \setminus T$, because the cycle necessarily contains other vertices than just the vertex $t$, that is, if there is a directed path from $t$ to $t$ that contains the vertex $v_i$, then there is a directed path from $v_i$ to $v_i$ as well. Furthermore, $v_i$ cannot be an unobserved vertex since unobserved vertices do not have parents.

Suppose first that the cycle does not include the vertex $t$. In this case exactly the same cycle exists in $\cG$, which contradicts the property that causal graphs are DAGs. On the other hand, if the cycle does contain the vertex $t$, then the cycle has a subpath $v_{j} \rightarrow t \rightarrow v_k$ such that $v_j \in \Pa{\cG}{T^+}$ and $v_k \in \Ch{\cG}{T}$ (if the cycle is $v_i \rightarrow t \rightarrow v_i$, then one lets $v_j = v_k = v_i$). Applying \cref{def_coherence} for $p = v_j$ and $c = v_k$ yields a directed cycle $v_i \rightarrow \cdots \rightarrow v_j \rightarrow t_1 \rightarrow \cdots \rightarrow t_m \rightarrow v_k \rightarrow \cdots \rightarrow v_i$ in $\cG$, where $t_1, \ldots, t_m \in T$. This again contradicts the property that $\cG$ is a DAG. 

\textbf{Case \ref{prop_coherence_one_maximal_subpath}:} Let us claim the opposite: there is a directed path $w_1 \rightarrow \cdots \rightarrow w_k$, where $w_1 \in V \cup U$ and $w_k \in V$, that intersects the set $T$ at least twice. In other words, there are at least two subpaths $w_{i_1} \rightarrow t_1 \rightarrow \cdots \rightarrow t_m \rightarrow w_{i_2}$ and $w_{j_1} \rightarrow t_{m + 1} \rightarrow \cdots \rightarrow t_{m + \ell} \rightarrow w_{j_2}$ in $\cG$ such that $w_{i_1}, w_{i_2}, w_{j_1}, w_{j_2} \notin T$ and $t_1, \ldots, t_m, t_{m + 1}, \ldots, t_{m + \ell} \in T$. However, then in $\cG'$ we have the path 
\begin{equation*}
w_1 \rightarrow \cdots \rightarrow w_{i_1} \rightarrow t \rightarrow w_{i_2} \rightarrow \cdots \rightarrow w_{j_1} \rightarrow t \rightarrow w_{j_2} \rightarrow \cdots \rightarrow w_k,
\end{equation*}
which contains a directed cycle $t \rightarrow \cdots \rightarrow t$, contradicting \ref{prop_coherence_no_cycles}.

\smallskip

\textbf{Case \ref{prop_coherence_ancestors}:} Let us fix a $w_j \in V \setminus T$. To show the first inclusion, we assume a vertex $w_i \in \Anc_{\cG}(w_j) \setminus T^+$. Now there is a directed path $w_i \rightarrow \cdots \rightarrow w_j$ in $\cG$. If the path does not contain the vertices of $T$, then exactly the same path exists in $\cG'$, and if the path does have at least one vertex of $T$, then by \ref{prop_coherence_one_maximal_subpath} it contains exactly one subpath $w_{i + k} \rightarrow t_1 \rightarrow \cdots \rightarrow t_m \rightarrow w_{j - \ell}$ such that $w_{i + k}, w_{j - \ell} \in V \setminus T$ and $t_1, \ldots, t_m \in T$. By the definition of clustering we can replace this subpath by $w_{i + k} \rightarrow t \rightarrow w_{j - \ell}$ to obtain a directed path from $w_i$ to $w_j$ in $\cG'$. This shows that $\Anc_{\cG}(w_j) \setminus T^+ \subset \Anc_{\cG'}(w_j) \setminus \left\{ t \right\}$.

To show the other inclusion, suppose that $w_i \in \Anc_{\cG'}(w_j) \setminus \{ t \}$. Directed paths from $w_i$ to $w_j$ that do not contain the vertex $t$ exist in $\cG$ as well. If there is a directed path $w_i \rightarrow \cdots \rightarrow w_j$ that contains the vertex $t$, then the path contains exactly one subpath $w_{i + k} \rightarrow t \rightarrow w_{j - \ell}$ such that $w_{i + k}, w_{j - \ell} \in V \setminus T$, because otherwise there would be a directed cycle $t \rightarrow \cdots \rightarrow t$. We replace this subpath by any path $w_{i + k} \rightarrow t_1 \rightarrow \cdots \rightarrow t_m \rightarrow w_{j - \ell}$ such that $t_1, \ldots, t_m \in T$, which is enabled by \cref{def_coherence}. This yields $\Anc_{\cG'}(w_j) \setminus \left\{ t \right\} \subset \Anc_{\cG}(w_j) \setminus T^+$.

The second assertion follows immediately, because there is a directed path from $w_i$ to $w_j$ if and only if $w_i$ is an ancestor of $w_j$
\end{proof}

Transit clusters \citep[see Appendix~\ref{apx_transit} for the definition]{thk2023} are coherent because by definition there is a directed path from any parent of the cluster to any child of the cluster.
However, the opposite does not apply, that is, all coherent clusters are not necessarily transit clusters. \Cref{fig_tr_cluster} shows an example of a transit cluster that indeed satisfies coherence, while \Cref{fig_coherent} presents a coherent cluster that is not a transit cluster. Comparing these two figures one observes that coherence requires significantly fewer edges than a transit cluster, but coherence also allows inclusion of bidirected arcs that would be prohibited by the conditions of a transit cluster, which is because transit clusters impose strong assumptions on incoming edges to the cluster.

On the other hand, coherence is not a property enforced in the settings of \citet{anand2023cdags,ferreira2025identifying,ferreira2026identifying} or \citet{yvernes2026relaxing}. This directly relates to the goal of this paper to maintain the identification invariance of the clustering operation. We focus on identification invariance via coherent clusters, which is discussed more explicitly in the following section.

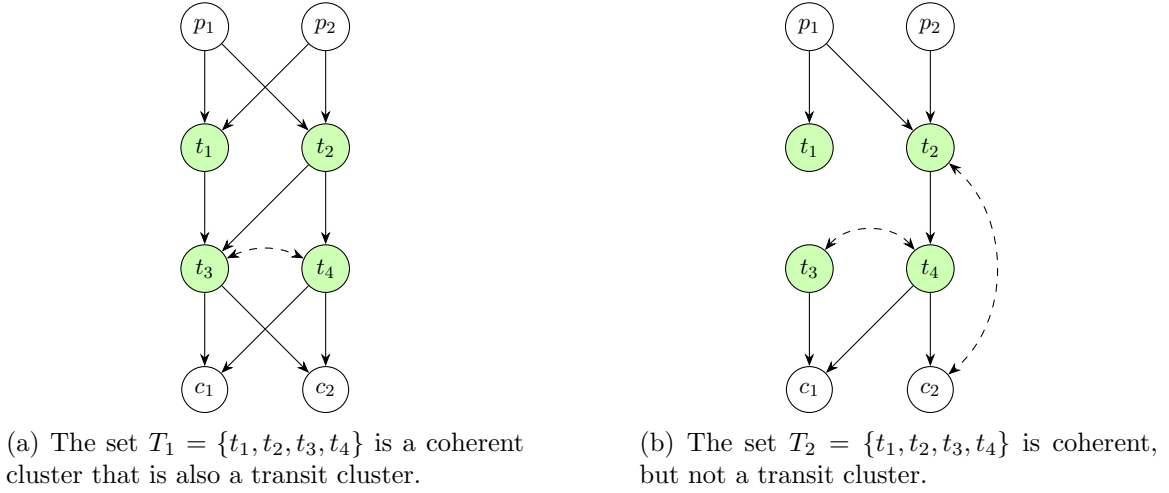
\begin{figure}
\centering
\subcaptionbox{The set $T_1 = \{ t_1, t_2, t_3, t_4 \}$ is a coherent cluster that is also a transit cluster.\label{fig_tr_cluster}}[0.45\textwidth]{

\begin{tikzpicture}[scale=0.8, transform shape, >={Stealth[scale=1]}]
\node[shape=circle, draw=black] (p1) at (0,0) {$p_1$};
\node[shape=circle, draw=black] (p2) at (2,0) {$p_2$};
\node[shape=circle, draw=black, fill=Cgreen] (t1) at (0,-2) {$t_1$};
\node[shape=circle, draw=black, fill=Cgreen] (t2) at (2,-2) {$t_2$};
\node[shape=circle, draw=black, fill=Cgreen] (t3) at (0,-4) {$t_3$};
\node[shape=circle, draw=black, fill=Cgreen] (t4) at (2,-4) {$t_4$};
\node[shape=circle, draw=black] (c1) at (0,-6) {$c_1$};
\node[shape=circle, draw=black] (c2) at (2,-6) {$c_2$};

\path [->] (t2) edge node[right] {} (t3);
\path [->] (p1) edge node[right] {} (t1);
\path [->] (p2) edge node[right] {} (t2);
\path [->] (p1) edge node[right] {} (t2);
\path [->] (p2) edge node[right] {} (t1);

\path [->] (t1) edge node[right] {} (t3);
\path [->] (t2) edge node[right] {} (t4);
\path [->] (t3) edge node[right] {} (c1);
\path [->] (t4) edge node[right] {} (c1);
\path [->] (t4) edge node[right] {} (c2);
\path [->] (t3) edge node[right] {} (c2);

\path [<->,dashed] (t3) edge[bend left=25] node[right] {} (t4);

\end{tikzpicture}
}
\hfill
\subcaptionbox{The set $T_2 = \{ t_1, t_2, t_3, t_4 \}$ is coherent, but not a transit cluster.\label{fig_coherent}}[0.45\textwidth]{

\begin{tikzpicture}[scale=0.8, transform shape, >={Stealth[scale=1]}]
\node[shape=circle, draw=black] (p1) at (0,0) {$p_1$};
\node[shape=circle, draw=black] (p2) at (2,0) {$p_2$};
\node[shape=circle, draw=black, fill=Cgreen] (t1) at (0,-2) {$t_1$};
\node[shape=circle, draw=black, fill=Cgreen] (t2) at (2,-2) {$t_2$};
\node[shape=circle, draw=black, fill=Cgreen] (t3) at (0,-4) {$t_3$};
\node[shape=circle, draw=black, fill=Cgreen] (t4) at (2,-4) {$t_4$};
\node[shape=circle, draw=black] (c1) at (0,-6) {$c_1$};
\node[shape=circle, draw=black] (c2) at (2,-6) {$c_2$};

\path [->] (p1) edge node[right] {} (t1);
\path [->] (p2) edge node[right] {} (t2);
\path [->] (p1) edge node[right] {} (t2);

\path [->] (t4) edge node[right] {} (c1);
\path [->] (t2) edge node[right] {} (t4);
\path [->] (t3) edge node[right] {} (c1);
\path [->] (t4) edge node[right] {} (c2);

\path [<->,dashed] (t2) edge[bend left=50] node[right] {} (c2);
\path [<->,dashed] (t3) edge[bend left=50] node[right] {} (t4);

\end{tikzpicture}
}
\caption{Examples of coherent clusters (\cref{def_coherence}).}
\label{fig_ex_coherence}
\end{figure}


\section{Identification Invariance of Clustering Operations} \label{sec_identif} 

\begin{toappendix}
  \label{apx_idproof}
\end{toappendix}

In this section we study clustering operations that are universally identification invariant, that is, clustering operations that preserve identifiability of those causal effects that do not involve the variables of the clustered set. We establish sufficient conditions for universal identification invariance of clustering operations based on the c-components of the original graph. In the end of this section we study identification invariance of iterated clustering operations. The proofs of this section are presented in Appendix~\ref{apx_idproof}. 

Throughout this section, we say that if $p(x_A \mid \Do(x_B))$ is a causal effect in a causal graph $\cG$, then the corresponding causal effect in the clustered graph $\cG'$ is $p'(x_A \mid \Do(x_B))$, provided that $A, B \subset V \setminus T$. We denote by $p$ and $p'$ distributions Markov-relative to $\cG$ and $\cG'$, respectively.

\smallskip

Do-calculus \citep{pearl_1} consists of operations that allow manipulation of interventional distributions. \citet{anand2023cdags} show that the rules of do-calculus are also valid in the clustered graph. This has the following consequences: First, if the clustering operation does not create directed cycles, then for each causal effect that is identifiable in the clustered graph, the corresponding causal effect is identifiable in the original graph. In particular this result holds if the clustering operation is coherent due to \cref{lemma_coherence_properties} \ref{prop_coherence_no_cycles}, but the result is also valid without the coherence assumption. Second, we can derive an identifying functional for the causal effect in the original graph from any identifying functional for the corresponding causal effect in the clustered graph. We present these results using our notation below. 

\begin{lem} \label{thm_identif_functional}
Let $A, B \subset V$ be disjoint sets in a causal graph $\cG = (V \cup U, E)$. Suppose a set of vertices $T \subset V \setminus (A \cup B)$ such that the graph $\cG' = (V' \cup U', E')$ induced by the clustering of $T$ is a DAG. If the causal effect $p'(x_A \mid \Do(x_B))$ is identifiable in $\cG'$ with an identifying functional $g(p'(x_{V'}))$ (that is, $p'(x_A \mid \Do(x_B)) = g(p'(x_{V'}))$), then $p(x_A \mid \Do(x_B))$ is identifiable in $\cG$ with a functional $g(p(x_V))$, which is obtained from $g(p'(x_{V'}))$ by replacing $x_t$ by the vector $x_T$. 
\end{lem}

An analogous result for transit clusters is proven by \citet[Theorem 33]{thk2023}.
\cref{thm_identif_functional} has the following implication: to obtain identification invariance of a clustering operation, it suffices to ensure that nonidentifiability of a causal effect in the clustered graph implies nonidentifiability of the same effect in the original graph. This is achieved by connecting identification invariance to the existence of hedges.

\begin{thmrep} \label{lemma_hedge_implies}
Assume a vertex set $T \subset V$ in a causal graph $\cG = (V \cup U, E)$ such that the graph $\cG' = (V' \cup U', E')$ induced by the clustering of $T$ is a DAG. Given disjoint sets $A, B \subset V \setminus T$, the following conditions are equivalent:
\begin{itemize}
    \item The clustering operation is identification invariant for $p(x_A \mid \Do(x_B))$.
    \item If there is a hedge for $p'(x_A \mid \Do(x_B))$ in $\cG'$, then there is a hedge for $p(x_A \mid \Do(x_B))$ in $\cG$.
\end{itemize}
\end{thmrep}

\begin{proof}
"$\implies$:" If there is a hedge for $p'(x_A \mid \Do(x_B))$ in $\cG'$, then by \cref{thm_hedge_criterion} the causal effect $p'(x_A \mid \Do(x_B))$ is nonidentifiable in $\cG'$. It follows from the definition of identification invariance and \cref{thm_hedge_criterion} that there is a hedge for the corresponding causal effect $p(x_A \mid \Do(x_B))$ in $\cG$.

\smallskip

"$\impliedby$:" If the causal effect $p'(x_A \mid \Do(x_B))$ is identifiable in $\cG'$, then by \cref{thm_identif_functional} the corresponding causal effect $p(x_A \mid \Do(x_B))$ is also identifiable in $\cG$. On the other hand, if $p'(x_A \mid \Do(x_B))$ is nonidentifiable in $\cG'$, then it follows from \cref{thm_hedge_criterion} that $p(x_A \mid \Do(x_B))$ is not identifiable in $\cG$ either.

If the causal effect $p(x_A \mid \Do(x_B))$ is identifiable in $\cG$, then there cannot be a hedge in $\cG$ by \cref{thm_hedge_criterion}, hence by the assumption there cannot be a hedge in $\cG'$, either, because that would imply the presence of the hedge in $\cG$. Finally, if $p(x_A \mid \Do(x_B))$ is nonidentifiable in $\cG$, then $p'(x_A \mid \Do(x_B))$ cannot be identifiable in $\cG'$, because that would violate \cref{thm_identif_functional}.
\end{proof}

According to \cref{lemma_hedge_implies}, to establish sufficient conditions for universal identification invariance of clustering operations, we need to find conditions such that clustering does not create hedges not present in the original graph. An example of a situation like this is provided in \Cref{fig_not_invariant}. To avoid such scenarios, we need additional assumptions on the c-components of the clusters. Sufficient conditions for transit clusters \citep[Definition 4]{thk2023} are proposed by \citet{thk2023}, and we aim to establish similar results for clusters that are merely coherent. As a reminder, since transit clusters are coherent, the results we establish in this section are also applicable to transit clusters. 

A transit cluster that does not have any unobserved parents (excluding those unobserved vertices whose both children are inside the clustered set) is called \emph{plain} according to \citet{thk2023}. This concepts generalizes directly to coherent clusters. 

\begin{defin}[Plainness] \label{def_plain}
A coherent cluster $T \subset V$ in a causal graph $\cG = (V \cup U, E)$ is \textit{plain} if $\Pa{\cG}{T^+} \subset V$, that is, the cluster does not have any unobserved parents that have one child in the cluster and one child outside the cluster.
\end{defin}

\begin{figure}[!t]
\centering
\subcaptionbox{Clustering the coherent cluster $T_1 = \left\{ t_1, t_2 \right\}$ does not preserve the identifiability of the causal effect ${p(x_a \mid \Do(x_b))}$.\label{fig_not_invariant}}[0.225\textwidth]{%
\begin{tikzpicture}[scale=0.8, transform shape, >={Stealth[scale=1]}]
\node[shape=circle, minimum size=0.80cm, draw=black] (b) at (0,0) {$b$};
\node[shape=circle, draw=black, fill=Cinvalid] (r) at (2,0) {$t_1$};
\node[shape=circle, draw=black, fill=Cinvalid] (e) at (2,-2) {$t_2$};
\node[shape=circle, minimum size=0.80cm, draw=black] (a) at (2,-4) {$a$};

\path [->] (b) edge node[right] {} (r);
\path [->] (r) edge node[right] {} (e);
\path [->] (e) edge node[right] {} (a);

\path [<->,dashed] (e) edge[bend left=50] node[right] {} (b);

\end{tikzpicture}
}
\hfill
\subcaptionbox{$T_2 = \left\{ t_1, t_2, t_3, t_4 \right\}$ is a plain cluster since there is no bidirected arc connecting $T_2$ to $V \setminus T$. \label{fig_plain}}[0.30\textwidth]{%
\begin{tikzpicture}[scale=0.8, transform shape, >={Stealth[scale=1]}]
\node[shape=circle,draw=black] (p1) at (0,0) {$b_1$};
\node[shape=circle,draw=black] (p2) at (2,0) {$b_2$};
\node[shape=circle,draw=black, fill=Cgreen] (r1) at (0,-2) {$t_1$};
\node[shape=circle,draw=black, fill=Cgreen] (r2) at (2,-2) {$t_2$};
\node[shape=circle,draw=black, fill=Cgreen] (e1) at (0,-4) {$t_3$};
\node[shape=circle,draw=black, fill=Cgreen] (e2) at (2,-4) {$t_4$};
\node[shape=circle,draw=black] (c1) at (0,-6) {$a_1$};
\node[shape=circle,draw=black] (c2) at (2,-6) {$a_2$};

\path [->] (p2) edge node[right] {} (r1);
\path [->] (p1) edge node[right] {} (r1);
\path [->] (p2) edge node[right] {} (r2);
\path [->] (r1) edge node[right] {} (e1);
\path [->] (r2) edge node[right] {} (e2);
\path [->] (e1) edge node[right] {} (c2);
\path [->] (e1) edge node[right] {} (c1);
\path [->] (e2) edge node[right] {} (c2);

\path [<->,dashed] (p2) edge[bend left=50] node[right] {} (c2);
\path [<->,dashed] (r2) edge[bend left=50] node[right] {} (e2);
\path [<->,dashed] (e1) edge[bend left=50] node[right] {} (e2);
\end{tikzpicture}
}
\hfill
\subcaptionbox{Every vertex of the cluster $T_3 = \left\{ t_1, t_2, t_3 \right\}$ belongs to the same c-component, but $T_3$ is not a transit cluster. Clustering the set $T_3$ renders the causal effect $p'(x_a \mid \doo(x_b))$ nonidentifiable in the clustered graph. \label{fig_congested_but_not_coherent}}[0.35\textwidth]{%
\begin{tikzpicture}[scale=0.8, transform shape, >={Stealth[scale=1]}]
\node[shape=circle, draw=black, fill=Cinvalid] (r) at (-2,0) {$t_1$};
\node[shape=circle,draw=black, fill=Cinvalid] (e1) at (-2,-2) {$t_2$};
\node[shape=circle,draw=black, fill=Cinvalid] (e2) at (0,-2) {$t_3$};
\node[shape=circle, minimum size=0.80cm, draw=black] (b) at (-2,-4) {$b$};
\node[shape=circle,draw=black] (v1) at (0,-4) {$v_1$};
\node[shape=circle, minimum size=0.80cm, draw=black] (a) at (-2,-6) {$a$};

\path [->] (r) edge node[right] {} (e1);
\path [->] (e1) edge node[right] {} (b);
\path [->] (b) edge node[right] {} (a);
\path [->] (e2) edge node[right] {} (v1);
\path [->] (v1) edge node[right] {} (b);
\path [->] (e1) edge node[right] {} (v1);
\path [->] (r) edge node[right] {} (e2);

\path [<->,dashed] (a) edge[bend left=50] node[right] {} (r);
\path [<->,dashed] (b) edge[bend left=30] node[right] {} (r);
\path [<->,dashed] (e1) edge[bend left=25] node[right] {} (e2);
\path [<->,dashed] (r) edge[bend left=50] node[right] {} (e2);

\end{tikzpicture}
}

\vspace{0.5cm}
\subcaptionbox{$T_4 = \left\{ t_1, t_2, t_3, t_4 \right\}$ is a c-coherent cluster with $t^* = t_1$. However, $T_4$ is not c-isolated due to the presence of a bidirected arc from $t_1$ to $a_1$. \label{fig_c_coherent}}[0.31\textwidth]{%
\begin{tikzpicture}[scale=0.8, transform shape, >={Stealth[scale=1]}]
\node[shape=circle,draw=black] (p1) at (-2,0) {$b_1$};
\node[shape=circle,draw=black] (p2) at (0,0) {$b_2$};
\node[shape=circle,draw=black] (p3) at (2,0) {$b_3$};
\node[shape=circle,draw=black, fill=Cgreen] (r1) at (0,-2) {$t_1$};
\node[shape=circle,draw=black, fill=Cgreen] (r2) at (2,-2) {$t_2$};
\node[shape=circle,draw=black, fill=Cgreen] (e1) at (0,-4) {$t_3$};
\node[shape=circle,draw=black, fill=Cgreen] (e2) at (2,-4) {$t_4$};
\node[shape=circle,draw=black] (c1) at (0,-6) {$a_1$};
\node[shape=circle,draw=black] (c2) at (2,-6) {$a_2$};

\path [->] (p3) edge node[right] {} (r1);
\path [->] (p2) edge node[right] {} (r1);
\path [->] (p3) edge node[right] {} (r2);
\path [->] (r1) edge node[right] {} (e1);
\path [->] (r2) edge node[right] {} (e2);
\path [->] (e2) edge node[right] {} (c1);
\path [->] (e1) edge node[right] {} (c1);
\path [->] (e2) edge node[right] {} (c2);
\path [->] (r1) edge node[right] {} (e2);
\path [->] (p1) edge node[right] {} (p2);

\path [<->,dashed] (c1) edge[bend left=50] node[right] {} (r1);
\path [<->,dashed] (r1) edge[bend left=50] node[right] {} (p1);
\path [<->,dashed] (r1) edge[bend left=30] node[right] {} (e2);
\end{tikzpicture}
}
\hfill
\subcaptionbox{$T_5 = \left\{ t_1, t_2, t_3, t_4 \right\}$ is a c-isolated cluster with $t^* = t_1$, but $T_5$ is not c-coherent, because $\cG[\{t_1, t_3\}^+]$ is not a c-component. \label{fig_c_isolated}}[0.31\textwidth]{%
\begin{tikzpicture}[scale=0.8, transform shape, >={Stealth[scale=1]}]
\node[shape=circle,draw=black] (p1) at (-2,0) {$b_1$};
\node[shape=circle,draw=black] (p2) at (0,0) {$b_2$};
\node[shape=circle,draw=black] (p3) at (2,0) {$b_3$};
\node[shape=circle,draw=black, fill=Cgreen] (r1) at (0,-2) {$t_1$};
\node[shape=circle,draw=black, fill=Cgreen] (r2) at (2,-2) {$t_2$};
\node[shape=circle,draw=black, fill=Cgreen] (e1) at (0,-4) {$t_3$};
\node[shape=circle,draw=black, fill=Cgreen] (e2) at (2,-4) {$t_4$};
\node[shape=circle,draw=black] (c1) at (0,-6) {$a_1$};
\node[shape=circle,draw=black] (c2) at (2,-6) {$a_2$};

\path [->] (p3) edge node[right] {} (r1);
\path [->] (p2) edge node[right] {} (r1);
\path [->] (p3) edge node[right] {} (r2);
\path [->] (r1) edge node[right] {} (e1);
\path [->] (r2) edge node[right] {} (e2);
\path [->] (e1) edge node[right] {} (c2);
\path [->] (e1) edge node[right] {} (c1);
\path [->] (e2) edge node[right] {} (c2);
\path [->] (p1) edge node[right] {} (p2);

\path [<->,dashed] (r1) edge[bend left=50] node[right] {} (p1);
\path [<->,dashed] (r1) edge[bend left=25] node[right] {} (e2);
\path [<->,dashed] (e1) edge[bend left=30] node[right] {} (e2);
\end{tikzpicture}
}
\hfill
\subcaptionbox{$T_6 = \left\{ t_1, t_2, t_3, t_4 \right\}$ is not c-isolated due to the lack of coherence. Accordingly identifiability of $p(x_{a_1} \mid \Do(x_{b_1}))$ is lost in clustering of $T_6$. \label{fig_not_c_isolated}}[0.31\textwidth]{%
\begin{tikzpicture}[scale=0.8, transform shape, >={Stealth[scale=1]}, every node/.style={minimum height=0.80cm}]
\node[shape=circle,draw=black] (b1) at (0,0) {$b_1$};
\node[shape=circle,draw=black] (b2) at (2,0) {$b_2$};
\node[shape=circle,draw=black, fill=Cinvalid] (t1) at (0,-2) {$t_1$};
\node[shape=circle,draw=black, fill=Cinvalid] (t2) at (2,-2) {$t_2$};
\node[shape=circle,draw=black, fill=Cinvalid] (t3) at (0,-4) {$t_3$};
\node[shape=circle,draw=black, fill=Cinvalid] (t4) at (2,-4) {$t_4$};
\node[shape=circle,draw=black] (a1) at (0,-6) {$a_1$};
\node[shape=circle,draw=black] (a2) at (2,-6) {$a_2$};

\path [->] (b1) edge node[right] {} (t1);
\path [->] (b2) edge node[right] {} (t1);
\path [->] (b1) edge node[right] {} (t2);
\path [->] (b2) edge node[right] {} (t2);
\path [->] (t2) edge node[right] {} (t4);
\path [->] (t4) edge node[right] {} (t3);
\path [->] (t4) edge node[right] {} (a2);
\path [->] (t3) edge node[right] {} (a1);
\path [->] (t4) edge node[right] {} (a1);
\path [->] (t4) edge node[right] {} (a2);

\path [<->,dashed] (t1) edge[bend left=50] node[right] {} (b1);

\end{tikzpicture}
}
\caption{Various causal graphs that demonstrate under which conditions universal identification invariance of the clustering operation holds (green), but also when this property is lost (red).}
\end{figure}
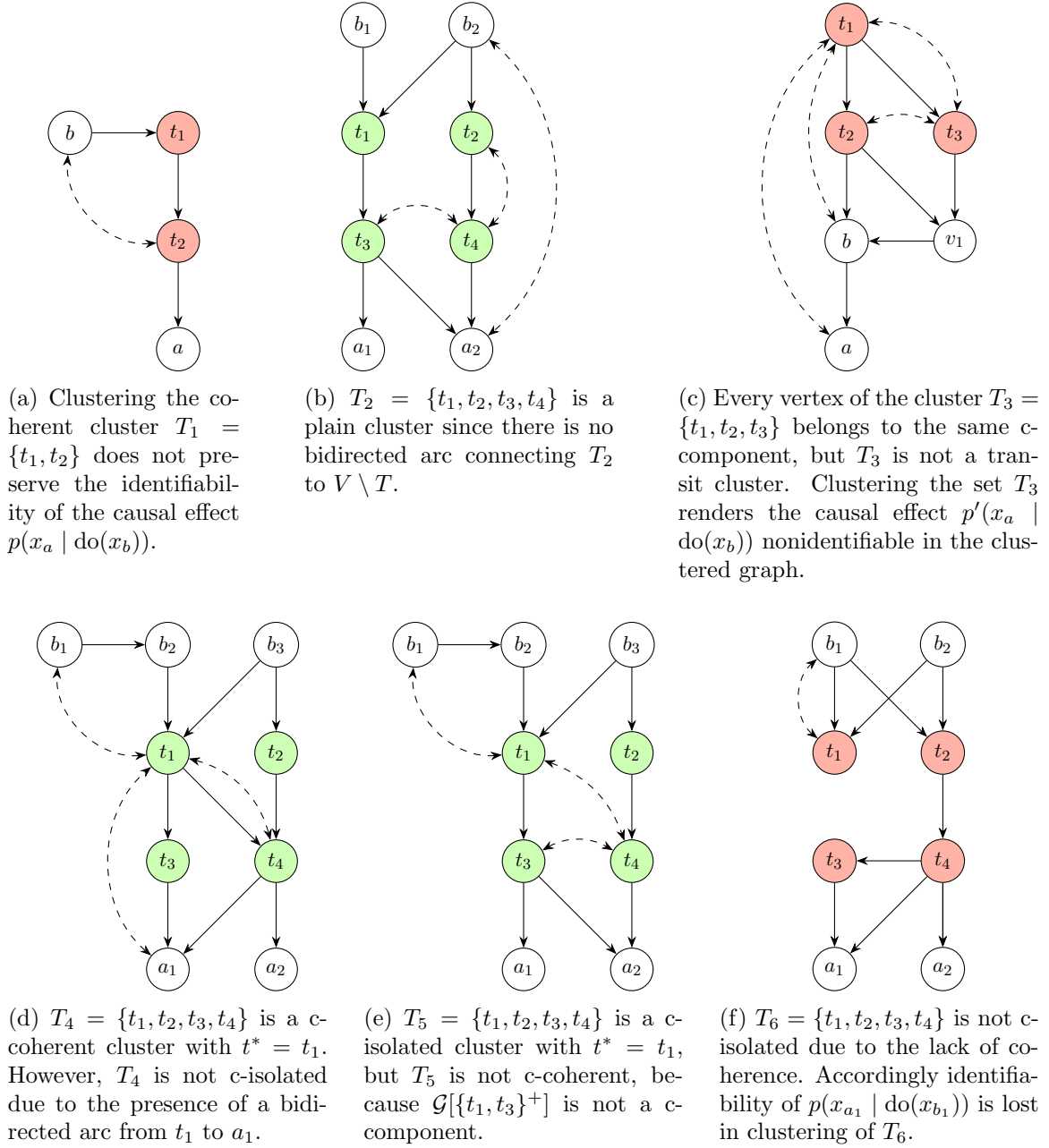

Plain clusters are characterized by the property that all bidirected arcs connected to the cluster are always disconnected from the vertices outside the cluster, implying that all c-components involving vertices of a plain cluster are strictly within in the cluster. An example of a plain cluster is provided in \Cref{fig_plain}.

We proceed to show that clustering a plain vertex set is a universally identification invariant operation.

\begin{thmrep} \label{thm_identif_plain}
Suppose a plain cluster $T \subset V$ in a causal graph $\cG = (V \cup U, E)$. Then clustering of $T$ is a universally identification invariant operation.
\end{thmrep}

\begin{proof}
Let us fix disjoint sets $A, B \subset V \setminus T$. In the following we denote by $\cG' = (V' \cup U', E')$ the graph induced by the clustering of $T$, and we let $t \in V'$ be the vertex representing the set $T$. According to \cref{lemma_hedge_implies}, it suffices to show that the presence of a hedge $(\cF', \cF_{*}')$ in $\cG'$ implies the presence of a corresponding hedge $(\cF, \cF_{*})$ in $\cG$.

We claim that  exactly the same hedge exists in $\cG$, that is, $(\cF, \cF_{*}) = (\cF', \cF_{*}')$ is a hedge for $p(x_A \mid \Do(x_B))$ in $\cG$. We show that the conditions \ref{prop_hedge_forests}--\ref{prop_hedge_root} of \cref{def_hedge} hold:
\begin{enumerate}[label=\rm{(\arabic*)}]

    \item There cannot be a bidirected arc connecting $V' \setminus \left\{ t \right\}$ to $\left\{ t \right\}$, because otherwise in $\cG$ there would be a fork $v_j \leftarrow u \rightarrow t_i$ such that $v_j \in V \setminus T$, $u \in U$ and $t_i \in T$, implying $\Pa{\cG}{T^+} \cap U \neq \emptyset$ and contradicting the plainness assumption. This has the following implications: all c-components in $\cG'$ (excluding the single vertex c-component that consists solely of the vertex $t$) are also present in $\cG$, and there cannot be a c-component that contains the vertex $t$ and some other vertex in $\cG'$. Thus, because $B$ is disjoint from $T$ and $V'_{\cF'} \cap B \neq \emptyset$, neither $\cF'$ nor $\cF_{*}'$ can contain the vertex $t$. It follows that $\cF = \cF'$ and $\cF_* = \cF_{*}'$ are c-forests in $\cG$.
    
    We still need to validate that $\cF$ and $\cF_*$ have the same root set. Let $R$ be the root set of the c-forests in $\cG$, and $R'$ the root set in $\cG'$. Since the vertex $t$ is excluded from the c-forests, it necessarily holds $R \cap T = R' \cap \left\{ t \right\} = \emptyset$, so applying \cref{lemma_coherence_properties} \ref{prop_coherence_ancestors} yields $\Anc_{\cG}(R) \setminus T = \Anc_{\cG'}(R') \setminus \left\{ t \right\}$. This verifies that $R$ is the root set of both $\cF$ and $\cF_{*}$.

    \item This is apparent because $V_{\cF} = V'_{\cF'}$ and $V_{\cF_*} = V'_{\cF_{*}'}$.
    
    \item According to \cref{lemma_subgraph_coherence} in Appendix~\ref{sec_appendix_aux} we have the inclusion 
    \begin{equation*}
    R = R \setminus T = R' \setminus \left\{ t \right\} \subset \Anc_{\cG'[\overline{B}]}(A) \setminus \left\{ t \right\} = \Anc_{\cG[\overline{B}]}(A) \setminus T \subset \Anc_{\cG[\overline{B}]}(A),
    \end{equation*}
    where we used the property $R' \subset \Anc_{\cG'[\overline{B}]}(A)$ to get the first inclusion.

\end{enumerate}
We conclude that $(\cF', \cF_{*}')$ is a hedge in $\cG$.
\end{proof}

\smallskip

\citet{thk2023} also define a congested transit cluster, which is a transit cluster where each vertex in the cluster belongs to the same c-component. However, for clusters that are merely coherent such a condition does not guarantee universal identification invariance of the clustering operation. This is demonstrated in \Cref{fig_congested_but_not_coherent}, which presents a coherent cluster whose members belong to the same c-component, and yet, the clustering operation is not universally identification invariant. Thus, instead of attempting to generalize the concept of a congested transit cluster, we introduce a comparable condition, which we call c-coherence.

\begin{defin}[c-coherence] \label{def_c_coherent}
A coherent cluster $T \subset V$ in a causal graph $\cG = (V \cup U, E)$ is \textit{c-coherent} if there exists a vertex $t^* \in T$ such that:
\begin{itemize}
    \item $\Pa{\cG}{t^*} \setminus T^+ = \Pa{\cG}{T^+}$, and
    \item for each $c \in \Ch{\cG}{T}$, it either applies that $t^* \in \Pa{\cG}{c}$, or there exists a directed path $t^* \rightarrow t_1 \rightarrow \cdots \rightarrow t_m \rightarrow c$ in $\cG$ such that $t_1, \ldots t_{m} \in T$ and the graph $\cG[P^+]$ induced by the vertex set $P = \left\{ t^*, t_1, \ldots, t_m \right\}$ is a c-component.
\end{itemize}
\end{defin}

The presence of the vertex $t^*$ in \cref{def_c_coherent} has the following implications: First, if there is an edge $w \rightarrow t$ in the clustered graph $\cG'$, where the vertex $t$ represents the cluster $T$, then in the original graph $\cG$ we have the edge $w \rightarrow t^*$. In particular this holds if $w \in U$. Consequently, all bidirected arcs connecting $V \setminus T$ to $T$ are always connected to $t^*$. An example of a c-coherent cluster is provided in \Cref{fig_c_coherent}, where $t^* = t_1$.

Second, the assumption on the paths from $t^*$ to $\Ch{\cG}{T}$ ensures that for each path $w \rightarrow t \rightarrow v$ in the clustered graph $\cG'$, where $w \in V \cup U$ and $v \in V$, there is a corresponding path $w \rightarrow t^* \rightarrow t_1 \rightarrow \cdots \rightarrow t_m \rightarrow v$ in $\cG$ such that $t_1, \ldots, t_m \in T$. Furthermore, if the original path belongs to the same c-component --- or even to the same c-forest --- in $\cG'$, then so does the corresponding path in $\cG$.

We next show that clustering a c-coherent cluster is also a universally identification invariant operation. 
\begin{thmrep} \label{thm_identif_c_coherent}
Suppose a c-coherent cluster $T \subset V$ in a causal graph $\cG = (V \cup U, E)$. Then the clustering of $T$ is a universally identification invariant operation.
\end{thmrep}

\begin{proof} 
Let $A, B \subset V \setminus T$ be disjoint. We denote by $\cG' = (V' \cup U', E')$ be the graph induced by the clustering of $T$, and by $t \in V'$ the vertex representing the set $T$. According to \cref{lemma_hedge_implies}, it suffices to show that if there is a hedge $(\cF', \cF_{*}')$ in $\cG'$, then we can always find a corresponding hedge $(\cF, \cF_{*})$ in $\cG$. We assume that the hedge in $\cG'$ contains the vertex $t$, because otherwise exactly the same hedge exists in $\cG$, which is proven similarly to in the proof of \cref{thm_identif_plain}. 

Let us denote the root set of $(\cF', \cF_{*}')$ by $R'$. We construct the hedge $(\cF, \cF_{*})$ as follows: If $t \in R'$, then we replace $t$ by the vertex $t^*$ provided by \cref{def_c_coherent} in both c-forests, which is enabled by the property that $\Pa{\cG}{t^*} \setminus T^+ = \Pa{\cG}{T^+}$. The root set in $\cG$ is then $R = (R' \setminus \left\{ t \right\}) \cup \left\{t^* \right\}$. On the other hand, if $t$ is not in the root set, then it has exactly one child $v_j \in V$ in $\cF'_*$. By \cref{def_c_coherent} there is a directed path $t^* \rightarrow t_1 \rightarrow \cdots \rightarrow t_m \rightarrow v_j$ such that $t_1, \ldots, t_m \in T$ and the graph $\cG[\left\{ t^*, t_1, \ldots, t_m \right\}^+]$ is a c-component. We replace the vertex $t$ by this c-component in $\cF'_*$. If $t$ is also present in $\cF'$, then we perform the same substitution in $\cF'$ as well. The root set remains the same, that is, $R = R'$.

We prove that $(\cF, \cF_{*})$ is a hedge by showing that the conditions \ref{prop_hedge_forests}--\ref{prop_hedge_root} of \cref{def_hedge} are satisfied:
\begin{enumerate}[label=\rm{(\arabic*)}]
    \item We observe that the assumption $\Pa{\cG}{t^*} \setminus T^+ = \Pa{\cG}{T^+}$ ensures that if there is a bidirected arc connecting $v_i \in V' \setminus \left\{ t \right\}$ to $\left\{ t \right\}$ in $\cG'$, then in $\cG$ there is a corresponding bidirected arc connecting $v_i$ to $t^*$. Furthermore, $t^*$ belongs to each of the c-components we formed by replacing $t$ by the directed path from $t^*$ to $t_m$. Bidirected arcs between vertices outside the cluster are unaffected. Consequently, $\cF$ and $\cF_{*}$ are c-components, and since in the substitutive path $t^* \rightarrow t_1 \rightarrow \cdots \rightarrow t_m$ each observed vertex has exactly one child, $\cF$ and $\cF_{*}$ are in fact c-forests.
    
    We still need to show that $R$ is the root set of the c-forests. Let us first consider the case $t \in R'$. In this case the only modification we made to the c-forests was to replace $t$ by $t^*$, and since $t^*$ has exactly the same parents as $t$, it is apparent that $R = (R' \setminus \left\{ t \right\}) \cup \left\{ t^* \right\}$ is the root set. In the case $t \notin R'$ it follows from \cref{lemma_coherence_properties} \ref{prop_coherence_ancestors} that $\Anc_{\cG}(v) \setminus T = \Anc_{\cG'}(v) \setminus \left\{ t \right\}$ for all $v \in V'_{\cF_{*}'} \setminus \left\{ t \right\} = V_{\cF_{*}} \setminus T$, which then implies that $R = R'$ is the root set of both c-forests also in $\cG$.

    \item Since $B$ is disjoint from $T$, it follows from the identities $V_{\cF} \setminus T = V'_{\cF'} \setminus \left\{ t \right\}$ and $V_{\cF_{*}} \setminus T = V'_{\cF_{*}'} \setminus \left\{ t \right\}$ that $V_{\cF} \cap B = \emptyset$, $V_{\cF_{*}} \cap B \neq \emptyset$ and $\cF \subset \cF_{*}$.
    
    \item We exploit \cref{lemma_subgraph_coherence} in Appendix~\ref{sec_appendix_aux} and the property $R' \subset \Anc_{\cG'[\overline{B}]}(A)$ to get
    \begin{equation*}
    R \setminus T = R' \setminus \left\{ t \right\} \subset \Anc_{\cG'[\overline{B}]}(A) \setminus \left\{ t \right\} = \Anc_{\cG[\overline{B}]}(A) \setminus T \subset \Anc_{\cG[\overline{B}]}(A).
    \end{equation*}
    If $R' \cap \left\{ t \right\} = \emptyset$, then also $ R\cap T = \emptyset$, so it follows immediately that $R \subset \Anc_{\cG[\overline{B}]}(A)$. On the other hand, if $t \in R'$, then $R \setminus \left\{ t^* \right\} = R' \setminus \left\{ t \right\} \subset \Anc_{\cG[\overline{B}]}(A)$. To show that $t^* \in \Anc_{\cG[\overline{B}]}(A)$, we first observe that because $t \in \Anc_{\cG'[\overline{B}]}(A)$, there is a directed path $t \rightarrow v_1 \rightarrow \cdots \rightarrow v_j \rightarrow a$ for some $a \in A$ in $\cG'[\overline{B}]$. Applying \cref{def_c_coherent} we find a corresponding path $t^* \rightarrow t_1 \rightarrow \cdots \rightarrow t_m \rightarrow v_i \rightarrow \cdots \rightarrow v_j \rightarrow a$ in $\cG$ for some $t_1, \ldots, t_m \in T$. Recalling that $B \cap T = \emptyset$, this path also exists in $\cG[\overline{B}]$. Thus, $t^* \in \Anc_{\cG[\overline{B}]}(A)$.
\end{enumerate}
We conclude that $(\cF, \cF_{*})$ is a hedge for the causal effect $p(x_A \mid \Do(x_B))$ in $\cG$.
\end{proof}

The definition of c-coherence has strong assumptions on the directed paths inside the cluster. These assumptions can be relaxed by requiring that there are no bidirected paths (i.e., paths that consist solely of bidirected arcs) from the vertex $t^*$ to the descendants of the cluster, a property which we call c-isolation.

\begin{defin}[c-isolation] \label{def_c_isolated}
A coherent cluster $T \subset V$ in a causal graph $\cG = (V \cup U, E)$ is \textit{c-isolated} if there exists a vertex $t^* \in T$ such that $\Pa{\cG}{t^*} \setminus T^+ = \Pa{\cG}{T^+}$, and for each c-component $\cD = (V^* \cup U^*, E^*)$ in $\cG$ it either holds $V^* \cap T = \emptyset$ or $V^* \cap \Dsc_{\cG}^*(T) = \emptyset$, that is, there is no bidirected path from $T$ to $\Dsc_{\cG}^*(T)$.
\end{defin}

In practice it suffices to verify that there is no bidirected path from $t^*$ to $\Dsc_{\cG}^*(T)$ in $\cG$ (or equivalently no bidirected path from $t$ to $\Dsc_{\cG'}^*(t)$ in $\cG'$), because each bidirected arc connecting $V \setminus T$ to $T$ necessarily connects to $t^*$ by the assumption $\Pa{\cG}{t^*} \setminus T^+ = \Pa{\cG}{T^+}$. 

An example of a c-isolated cluster is provided in \Cref{fig_c_isolated}. The set $T_5 = \{ t_1, t_2, t_3, t_4 \}$ is c-isolated, because there is no bidirected path connecting $t^* = t_1$ to the set $\Dsc_{\cG}^*(T_5) = \{ a_1, a_2 \}$, but $T_5$ is not c-coherent, because the only directed path from $t^* = t_1$ to $a_2$ is $t_1 \rightarrow t_3 \rightarrow a_2$, but the subgraph $\cG[\{ t_1, t_3 \}^+]$ is not a c-component. In \Cref{fig_c_coherent}, the vertex set $T_4$ is not c-isolated, because there is a bidirected arc from $t_1$ to $a_1 \in \Dsc_{\cG}^*(T_4)$. However, if one removes this arc, then $T_4$ becomes both c-coherent and c-isolated, illustrating that these two conditions are not mutually exclusive. 

If in the definition of coherence (\cref{def_coherence}) one requires that only the \emph{observed} parents $\Pa{\cG}{T^+} \cap V$ are connected to $\Ch{\cG}{T}$ via specific paths, then c-isolation no longer guarantees universal identification invariance. This can be seen in \Cref{fig_not_c_isolated}, where each directed path from the unobserved vertex denoted by the bidirected arc between $t_1$ and $b_1$ goes through the vertex $b_1 \in V \setminus T$. Even though there is no bidirected paths from $T_6$ to its descendants, the clustering operation renders the causal effect $p(x_{a_1} \mid \Do(x_{b_2}))$ nonidentifiable in the clustered graph.

The similarities and differences of c-coherence and c-isolation can be summarised as follows: Unlike plainness (\cref{def_plain}), both c-coherence and c-isolation allow bidirected arcs connecting $T$ and $V \setminus T$, and in both conditions all bidirected arcs connected to $T$ have to connect to a single vertex $t^*$. Unlike c-isolation, c-coherence allows bidirected arcs from $T$ to its descendants, but one needs additional assumptions on how the vertex $t^*$ connects to the children of the cluster, because otherwise clustering can create a hedge not present in the original graph. For c-isolated clusters one does not need to know anything about the paths inside the cluster, because c-isolation always ensures that if the vertex $t$ representing the cluster $T$ in the clustered graph belongs to a hedge, then $t$ is necessarily in the root set.

We proceed to show that clustering a c-isolated vertex set is a universally identification invariant operation.

\begin{thmrep} \label{thm_identif_c_isolation}
Suppose a c-isolated set $T \subset V$ in a causal graph $\cG = (V \cup U, E)$. Then clustering of $T$ is a universally identification invariant operation.
\end{thmrep}

\begin{proof}
Let $A, B \subset V \setminus T$ be disjoint sets. We denote by $\cG' = (V' \cup U', E')$ the graph induced by clustering of $T$, and we let $t \in V'$ be the vertex representing the set $T$. According to \cref{lemma_hedge_implies}, it suffices to show that the presence of a hedge $(\cF', \cF_{*}')$ in $\cG'$ implies the presence of a corresponding hedge $(\cF, \cF_{*})$ in $\cG$. If the hedge in $\cG'$ does not contain the vertex $t$, then exactly the same hedge exists in $\cG$, which can be seen using a similar argument as in the proof of \cref{thm_identif_plain}.

Let us then suppose that the vertex $t$ is present at least in the larger c-forest $\cF_*'$. Because there cannot be a c-component which contains both $t$ and its descendants, $t$ is necessarily in the root set $R'$. We construct a hedge $(\cF, \cF_{*})$ in $\cG$ by replacing $t$ by $t^*$ in $\cF_*'$, where it is necessarily present, and also in $\cF'$ in the case $t \in V'_{\cF'}$. The root set is then $R = (R' \setminus \left\{ t \right\}) \cup \left\{ t^* \right\}$. This construction yields a hedge because we only modify the root set by replacing one vertex with another one with similar properties, i.e., all bidirected arcs connecting to $t$ also connect to $t^*$ by the assumption that $\Pa{\cG}{t^*} \setminus T^+ = \Pa{\cG}{T^+}$.
\end{proof}


The three conditions presented above, namely plainness, c-coherence and c-isolation, do not provide complete conditions to assess universal identification invariance. Instead, the purpose is to provide sufficient conditions that are relatively straightforward to verify from the original graph. To obtain a complete condition, one would need to impose conditions specific to each graph --- or alternatively apply \cref{lemma_hedge_implies} and verify whether a presence of a hedge in the clustered $\cG'$ always implies a presence of a hedge in the original graph $\cG$ for each causal effect not involving the variables of the clustered set.

\begin{figure}[!t]
\centering
\subcaptionbox{$T_1 = \{ t_1, t_2\}$ satisfies other conditions of plainness, but not coherence. \label{fig_almost_plain}}[0.45\textwidth]{

\begin{tikzpicture}[scale=0.8, transform shape, >={Stealth[scale=1]}]
\node[shape=circle, draw=black] (b1) at (0,0) {$b_1$};
\node[shape=circle, draw=black] (b2) at (2,0) {$b_2$};
\node[shape=circle, draw=black, fill=Cgreen] (t1) at (0,-2) {$t_1$};
\node[shape=circle, draw=black, fill=Cgreen] (t2) at (2,-2) {$t_2$};
\node[shape=circle, draw=black] (a1) at (0,-4) {$a_1$};
\node[shape=circle, draw=black] (a2) at (2,-4) {$a_2$};

\path [->] (b2) edge node[right] {} (b1);
\path [->] (b1) edge node[right] {} (t1);
\path [->] (t1) edge node[right] {} (a1);
\path [->] (b2) edge node[right] {} (t2);
\path [->] (t2) edge node[right] {} (a2);
\path [->] (a1) edge node[right] {} (a2);

\path [<->,dashed] (b1) edge[bend left=50] node[right] {} (b2);
\path [<->,dashed] (b2) edge[bend left=50] node[right] {} (a2);
\path [<->,dashed] (t1) edge[bend left=50] node[right] {} (t2);

\end{tikzpicture}
}
\hfill
\subcaptionbox{The coherent cluster $T_2 = \{ t_1, t_2, t_3, t_4 \}$ is not c-coherent or c-isolated, because there are two different vertices inside the cluster that are connected to vertices outside the cluster via bidirected arcs.\label{fig_almost_c_coherent}}[0.45\textwidth]{
    
\begin{tikzpicture}[scale=0.8, transform shape, >={Stealth[scale=1]}]
\node[shape=circle, draw=black] (b1) at (0,0) {$b_1$};
\node[shape=circle, draw=black] (b2) at (2,0) {$b_2$};

\node[shape=circle, draw=black, fill=Cgreen] (t1) at (0,-2) {$t_1$};
\node[shape=circle, draw=black, fill=Cgreen] (t2) at (2,-2) {$t_2$};
\node[shape=circle, draw=black, fill=Cgreen] (t3) at (0,-4) {$t_3$};
\node[shape=circle, draw=black, fill=Cgreen] (t4) at (2,-4) {$t_4$};

\node[shape=circle, draw=black] (a1) at (0,-6) {$a_1$};
\node[shape=circle, draw=black] (a2) at (2,-6) {$a_2$};

\path [->] (b1) edge node[right] {} (t1);
\path [->] (t3) edge node[right] {} (a1);
\path [->] (b2) edge node[right] {} (t2);
\path [->] (t4) edge node[right] {} (a2);
\path [->] (t3) edge node[right] {} (t4);
\path [->] (t2) edge node[right] {} (t1);
\path [->] (t2) edge node[right] {} (t4);
\path [->] (t1) edge node[right] {} (t3);
\path [->] (b2) edge node[right] {} (t1);

\path [<->,dashed] (t1) edge[bend left=50] node[right] {} (b1);
\path [<->,dashed] (a1) edge[bend left=50] node[right] {} (t1);
\path [<->,dashed] (t3) edge[bend left=30] node[right] {} (t1);
\path [<->,dashed] (t4) edge[bend left=50] node[right] {} (t3);
\path [<->,dashed] (t2) edge[bend left=50] node[right] {} (a2);
\end{tikzpicture}
}
\caption{Examples of clusters (marked in green) for which the corresponding clustering operation is universally identification invariant, but none of the clusters is plain, c-coherent or c-isolated.} \label{fig_almost_applicable}
\end{figure}
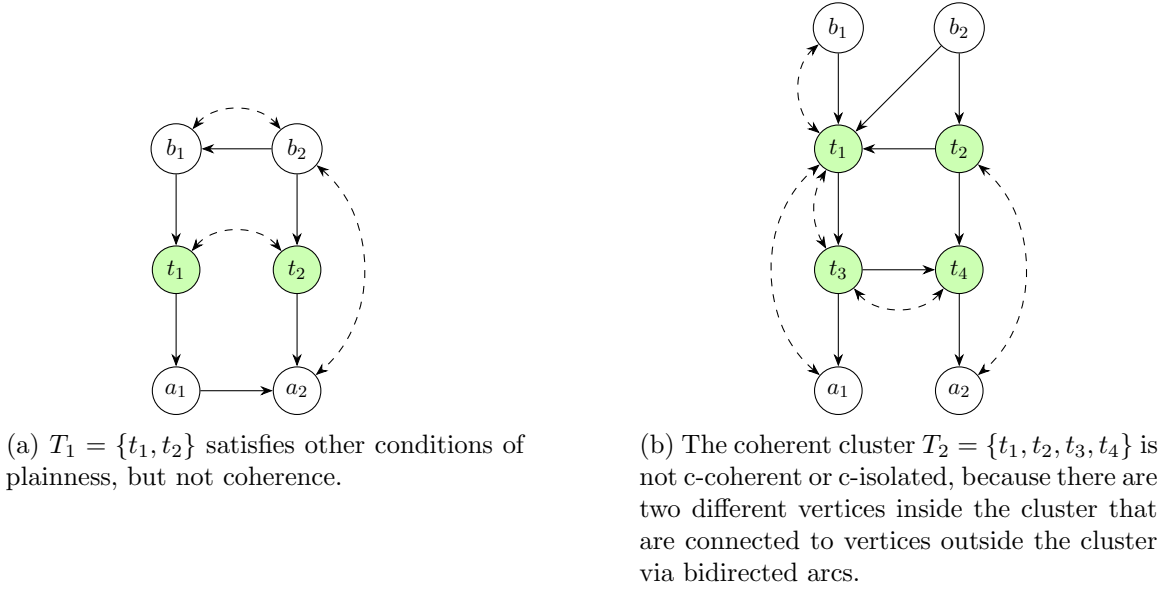

Examples of clustering operations whose universal identification invariance cannot be deduced using the graphical conditions from \cref{sec_identif} are presented in \Cref{fig_almost_applicable}, where clustering of the vertex sets marked in green are universally identification invariant, but neither cluster is plain, c-coherent or c-isolated. The cluster $T_1 = \{ t_1, t_2 \}$ in \Cref{fig_almost_plain} is not plain due to the lack of coherence, and while the cluster $T_2 = \{ t_1, \ldots, t_4 \}$ in \Cref{fig_almost_c_coherent} is coherent, it is neither c-coherent nor c-isolated because there are bidirected arcs connecting vertices outside the cluster to two different vertices of the cluster.

\smallskip
The following result shows how identification invariance allows us to make conclusions about the nonidentifiability of causal effects.
\begin{thmrep} \label{thm_nonid}
Suppose a cluster $T \subset V$ in a causal graph $\cG = (V \cup U, E)$ such that clustering of $T$ in $\cG$ induces a DAG. Let $\cH$ be a graph obtained from $\cG$ by removing a subset 
$U_S \subseteq U$, the edges connected to $U_S$, and a subset of edges $v_i \rightarrow v_j$, $v_i,v_j \in V$,
so that clustering of $T$ is a universally identification invariant operation in $\cH$. Let $\cH'$ be the graph induced by clustering of $T$ in $\cH$. If a causal effect $p'(x_A \cond \doo(x_B))$, where $A,B \subset V \setminus T$, is not identifiable in $\cH'$, then $p(x_A \cond \doo(x_B))$ is not identifiable in $\cG$.
\end{thmrep}
\begin{proof}
As clustering of $T$ is a universally identification invariant operation in $\cH$, nonidentifiability in $\cH'$ implies nonidentifiability in $\cH$. 
The original graph $\cG$ can be obtained from $\cH$ by adding the vertices $U_S$ and the edges $E_S$ that were removed.
Inserting latent variables or new edges to $\cH$ cannot render a nonidentifiable causal effect identifiable, and therefore the causal effect $p(x_A \cond \doo(x_B))$ is not identifiable from $p(x_V)$ in $\cG$. 
\end{proof}

As an example of application of \Cref{thm_nonid}, we consider again \Cref{fig_almost_applicable}. Removing the edges $b_2 \rightarrow t_2$ and $t_2 \rightarrow a_2$ in \Cref{fig_almost_plain} makes the vertex set $T_1 = \{ t_1, t_2 \}$ coherent and plain. We observe that each causal effect of $X_{b_2}$ on the variables $X_{a_1}$, $X_{a_2}$ and $X_{b_1}$ is nonidentifiable in the graph obtained by removing the aforementioned edges and clustering of $T_2$. It follows from \cref{thm_identif_plain,thm_nonid} that the corresponding causal effects are also nonidentifiable in the original graph. 

In \Cref{fig_almost_c_coherent}, the cluster $T_2$ becomes c-coherent if the bidirected arc between $t_2$ and $a_2$ is removed. Studying the graph where the bidirected arc is removed and the set $T_2$ is clustered, we see that the causal effects of $X_{b_1}$ on the variables $X_{a_1}$ and $X_{a_2}$ are not identifiable. \Cref{thm_identif_c_coherent,thm_nonid} imply that the corresponding causal effects are not identifiable in the original graph either. 

\smallskip

We next show that if one applies multiple successive clustering operations that are identification invariant for a given causal effect, then the total clustering operation is also identification invariant for the given effect.

\begin{thmrep}[Iterative clustering] \label{thm_ident_multiple}
Let $A, B \subset V$ be disjoint sets in a causal graph $\cG = (V \cup U, E)$. Suppose a family of graphs $\cG^{(i)} = (V^{(i)} \cup U^{(i)}, E^{(i)})$ for $i = 0, \ldots, N$, and vertex sets $T^{(i)} \subset V^{(i)} \setminus (A \cup B)$ such that $\cG^{(0)} = \cG$ and $\cG^{(i)}$ is obtained from $\cG^{(i - 1)}$ by clustering the set $T^{(i)}$ for all $i = 1, \ldots, N$. Denote by $p^{(i)}$ a distribution Markov relative to $\cG^{(i)}$ such that $p^{(i)}(x_{V^{(i)}}) > 0$, and assume that for all $i = 1, \ldots, N$
\begin{itemize}
    \item the graph $\cG^{(i)}$ is a DAG, and 
    \item the operation $\Omega_{T^{(i)}}(\cG^{(i - 1)}) = \cG^{(i)}$ that clusters the set $T^{(i)}$ in the graph $\cG^{(i - 1)}$ is an identification invariant operation for the causal effect $p^{(i - 1)}(x_A \mid \Do(x_B))$.
\end{itemize}
Then the iterated clustering operation $\Omega(\cG) := (\Omega_{T^{(N)}} \circ \Omega_{T^{(N - 1)}} \circ \cdots \circ \Omega_{T^{(1)}}) (\cG) = \cG^{(N)}$ is identification invariant for the causal effect $p(x_A \mid \Do(x_B))$.
\end{thmrep}

\begin{proof}
We prove the claim by induction. In the case $N = 1$ the iterated clustering operation corresponds to clustering a single set, hence the identification invariance follows from the assumptions. Next suppose that the assertion holds for $N = k$, that is, the causal effect $p(x_A \mid \Do(x_B))$ is identifiable in $\cG$ if and only if the corresponding causal effect $p^{(k)}(x_A \mid \Do(x_B))$ identifiable in $\cG^{(k)}$. On the other hand, the operation that maps $\cG^{(k)}$ to $\cG^{(k + 1)}$ is assumed to be identification invariant for $p^{(k)}(x_A \mid \Do(x_B))$, that is, the causal effect $p^{(k)}(x_A \mid \Do(x_B))$ is identifiable in $\cG^{(k)}$ if and only if $p^{(k + 1)}(x_A \mid \Do(x_B))$ is identifiable in $\cG^{(k + 1)}$. Combining these two equivalence relations implies that the iterated operation $\Omega(\cG)$ is identification invariant for the given causal effect.
\end{proof}

Consequently, iterating multiple universally identification invariant clustering operations is also a universally identification invariant operation, as shown in the following corollary.

\begin{cororep} \label{coro_iterated_universal}
Suppose an iterated clustering operation
\begin{equation*}
\Omega(\cG) = (\Omega_{T^{(N)}} \circ \Omega_{T^{(N - 1)}} \circ \cdots \circ \Omega_{T^{(1)}}) (\cG) = \cG^{(N)}
\end{equation*}
defined as in \cref{thm_ident_multiple}. If the operation $\Omega_{T^{(i)}}$ is universally identification invariant for all $i = 1, \ldots, N$, then the iterated operation $\Omega(\cG)$ is identification invariant for each causal effect not involving variables $X_{T^{(j)}}$ for any $j = 1, \ldots, N$.
\end{cororep}

\begin{proof}
Each individual clustering operation $\Omega_{T^{(i)}}$ is identification invariant for all causal effects not involving the variables of the corresponding cluster $X_{T^{(i)}}$, so in particular the operation $\Omega_{T^{(i)}}$ is identification invariant for all causal effects not involving variables of any $X_{T^{(j)}}$, $j = 1, \ldots, N$. The assertion then follows from \cref{thm_ident_multiple}.
\end{proof}

An example how \cref{thm_ident_multiple} can be applied in practice is provided in \cref{sec_app}.


\section{Demonstration} \label{sec_app}

In this section we demonstrate how one can apply the identification results established in \cref{sec_identif} in practice. We use a causal graph adapted from \citet{evandt07noise}, but for illustrative purposes we have included bidirected arcs not present in the original graph. Explanation for different variables in the model is provided in \Cref{table_sleepmed} in Appendix~\ref{sec_appendix_table}.

\begin{figure}[!t]
\centering
\subcaptionbox{The original graph $\cG = (V, E)$. \label{fig_sleepmed_orig}}[0.65\textwidth]{

\begin{tikzpicture}[scale=0.8, transform shape, >={Stealth[scale=1]}, every node/.style={minimum height=0.80cm}]

\node[shape=circle, draw=black] (d) at (0, -2) {$d$};
\node[shape=circle, draw=black] (p) at (-8, -4) {$p$};
\node[shape=circle, draw=black] (e) at (-6,-2) {$e$};

\node[shape=circle, draw=black, fill=Cblue] (s) at (-2,-5) {$s$};
\node[shape=circle, draw=black, fill=Cblue] (c) at (0,-5) {$c$};
\node[shape=circle, draw=black, fill=Cblue] (l) at (2,-5) {$\ell$};

\node[shape=circle, draw=black, fill=Cgreen] (n1) at (-8,-8) {$n_1$};
\node[shape=circle, draw=black, fill=Cgreen] (n2) at (-6,-6) {$n_2$};
\node[shape=circle, draw=black, fill=Cgreen] (n3) at (-4,-4) {$n_3$};
\node[shape=circle, draw=black, fill=Cgreen] (a) at (-4,-6) {$a$};

\node[shape=circle, draw=black] (m) at (0,-8) {$m$};

\path [->] (d) edge node[right] {} (e);
\path [->] (d) edge node[right] {} (s);
\path [->] (d) edge node[right] {} (c);
\path [->] (d) edge node[right] {} (l);
\path [->] (e) edge node[right] {} (p);
\path [->] (e) edge node[right] {} (n2);
\path [->] (e) edge node[right] {} (n3);
\path [->] (p) edge node[right] {} (n1);
\path [->] (p) edge node[right] {} (n2);
\path [->] (s) edge node[right] {} (a);
\path [->] (s) edge node[right] {} (m);
\path [->] (c) edge node[right] {} (m);
\path [->] (l) edge node[right] {} (c);
\path [->] (l) edge node[right] {} (m);
\path [->] (a) edge node[right] {} (m);
\path [->] (n2) edge node[right] {} (a);
\path [->] (n3) edge node[right] {} (a);
\path [->] (n1) edge node[right] {} (m);
\path [->] (n2) edge node[right] {} (m);
\path [->] (n3) edge node[right] {} (m);

\path [<->,dashed,draw=blue] (e) edge[bend left=50] node[] {} (d);
\path [<->,dashed,draw=blue] (s) edge[bend left=40] node[] {} (d);
\path [<->,dashed,draw=blue] (s) edge[bend left=50] node[] {} (n2);

\path [<->,dashed,draw=blue] (n1) edge[bend left=25] node[] {} (n2);
\path [<->,dashed,draw=blue] (n2) edge[bend left=25] node[] {} (n3);
\path [<->,dashed,draw=blue] (s) edge[bend left=25] node[] {} (c);
\path [<->,dashed,draw=blue] (m) edge[bend left=50] node[] {} (p);

\end{tikzpicture}
}
\hfill
\subcaptionbox{A graph $\cG'$ obtained from $\cG$ by clustering the c-isolated vertex set $T_1 = \{ a, n_1, n_2, n_3 \}$ (here $t^* = n_2$).\label{fig_sleepmed_partially_clustered}}[0.375\textwidth]{

\begin{tikzpicture}[scale=0.8, transform shape, >={Stealth[scale=1]}, every node/.style={minimum height=0.80cm}]

\node[shape=circle, draw=black] (d) at (4, 0) {$d$};
\node[shape=circle, draw=black] (e) at (2, 0) {$e$};
\node[shape=circle, draw=black] (p) at (0, 0) {$p$};

\node[shape=circle, draw=black, fill=Cgreen] (t1) at (0,-2) {$t_1$};

\node[shape=circle, draw=black, fill=Cblue] (s) at (2,-2) {$s$};
\node[shape=circle, draw=black, fill=Cblue] (c) at (4,-2) {$c$};
\node[shape=circle, draw=black, fill=Cblue] (l) at (6,-2) {$\ell$};

\node[shape=circle, draw=black] (m) at (2,-4) {$m$};

\path [->] (d) edge node[right] {} (e);
\path [->] (d) edge node[right] {} (s);
\path [->] (d) edge node[right] {} (c);
\path [->] (d) edge node[right] {} (l);

\path [->] (s) edge node[right] {} (t1);
\path [->] (s) edge node[right] {} (m);
\path [->] (c) edge node[right] {} (m);
\path [->] (l) edge node[right] {} (c);
\path [->] (l) edge node[right] {} (m);

\path [->] (e) edge node[right] {} (p);
\path [->] (e) edge node[right] {} (t1);
\path [->] (p) edge node[right] {} (t1);
\path [->] (t1) edge node[right] {} (m);

\path [<->,dashed,draw=blue] (e) edge[bend left=50] node[] {} (d);
\path [<->,dashed,draw=blue] (s) edge[bend left=25] node[] {} (d);
\path [<->,dashed,draw=blue] (s) edge[bend left=25] node[] {} (t1);
\path [<->,dashed,draw=blue] (s) edge[bend left=25] node[] {} (c);
\path [<->,dashed,draw=blue] (m) edge[bend left=75] node[] {} (p);

\end{tikzpicture}
}
\hfill
\subcaptionbox{A graph $\cG''$ obtained from $\cG'$ clustering the c-coherent vertex set $T_2 = \{ c, \ell, s  \}$ in $\cG'$ (notice that $T_2$ is c-isolated in $\cG$, but not in $\cG'$). \label{fig_sleepmed_clustered}}[0.275\textwidth]{

\begin{tikzpicture}[scale=0.8, transform shape, >={Stealth[scale=1]}, every node/.style={minimum height=0.80cm}]

\node[shape=circle, draw=black] (d) at (4, 0) {$d$};
\node[shape=circle, draw=black] (e) at (2, 0) {$e$};
\node[shape=circle, draw=black] (p) at (0, 0) {$p$};

\node[shape=circle, draw=black, fill=Cgreen] (t1) at (1,-2) {$t_1$};
\node[shape=circle, draw=black, fill=Cblue] (t2) at (3,-2) {$t_2$};

\node[shape=circle, draw=black] (m) at (2,-4) {$m$};

\path [->] (d) edge node[right] {} (e);
\path [->] (d) edge node[right] {} (t2);
\path [->] (e) edge node[right] {} (p);
\path [->] (e) edge node[right] {} (t1);
\path [->] (p) edge node[right] {} (t1);
\path [->] (t2) edge node[right] {} (m);
\path [->] (t1) edge node[right] {} (m);
\path [->] (t2) edge node[right] {} (t1);

\path [<->,dashed,draw=blue] (e) edge[bend left=50] node[] {} (d);
\path [<->,dashed,draw=blue] (d) edge[bend left=50] node[] {} (t2);
\path [<->,dashed,draw=blue] (t2) edge[bend left=25] node[] {} (t1);
\path [<->,dashed,draw=blue] (m) edge[bend left=50] node[] {} (p);

\end{tikzpicture}
}
\hfill
\subcaptionbox{A graph $\cG^{(3)} = (V^{(3)}, E^{(3)})$ obtained from $\cG''$ by clustering the c-isolated vertex set $T_3 = \{ t_1, t_2 \}$. \label{fig_sleepmed_clustered_final}}[0.275\textwidth]{

\begin{tikzpicture}[scale=0.8, transform shape, >={Stealth[scale=1]}, every node/.style={minimum height=0.80cm}]

\node[shape=circle, draw=black] (d) at (4, 0) {$d$};
\node[shape=circle, draw=black] (e) at (2, 0) {$e$};
\node[shape=circle, draw=black] (p) at (0, 0) {$p$};

\node[shape=semicircle, draw=Cgreen, fill=Cgreen, rotate=90, scale=0.5] () at (1.825,-2) {};
\node[shape=semicircle, draw=Cblue, fill=Cblue, rotate=270, scale=0.5] () at (2.175,-2) {};
\node[shape=circle, draw=black] (t3) at (2,-2) {$t_3$};

\node[shape=circle, draw=black] (m) at (2,-4) {$m$};

\path [->] (d) edge node[right] {} (e);
\path [->] (d) edge node[right] {} (t3);
\path [->] (e) edge node[right] {} (p);
\path [->] (e) edge node[right] {} (t3);
\path [->] (p) edge node[right] {} (t3);
\path [->] (t3) edge node[right] {} (m);

\path [<->,dashed,draw=blue] (e) edge[bend left=50] node[] {} (d);
\path [<->,dashed,draw=blue] (d) edge[bend left=50] node[] {} (t3);
\path [<->,dashed,draw=blue] (m) edge[bend left=50] node[] {} (p);

\end{tikzpicture}
}
\caption{Graphs modeling how different factors affect sleep medication use (adapted from \citep{evandt07noise}, additional bidirected arcs included). Color-coding is applied to highlight different clusters of interest: Vertices belonging to the cluster $T_1$ and the cluster $T_2$ are marked in green and blue, respectively. Bidirected arcs are highlighted with blue for clarity.} \label{fig_sleepmed}
\end{figure}

In practice, we are typically interested in a certain family of causal effects. Accordingly, the clustering operations should not contain variables of the causal effects of interest. In this demonstration, we assume that we are interested in how demography, socioeconomic status and population density ($X_d$, $X_e$ and $X_p$, respectively) affect sleep medication use $X_m$. The family of causal effects we are interested in is then the set 
\begin{align*}
\Big\lbrace p(x_m \mid \Do(x_d)), \, p(x_m \mid \Do(x_e)), \, p(x_m \mid \Do(x_p)), \, p(x_m \mid \Do(x_d, x_e)), \\
p(x_m \mid \Do(x_d, x_p)), \, p(x_m \mid \Do(x_e, x_p)), \, p(x_m \mid \Do(x_d, x_e, x_p)) \bigg\rbrace.
\end{align*} 

Based on the structure of the graph $\cG$ in \Cref{fig_sleepmed_orig}, we propose the following candidates for clustered sets: $T_1 = \{ a, n_1, n_2, n_3 \}$ and $T_2 = \{ c, \ell, s \}$. While other clustering options are possible, these particular choices enable us to apply identifiability results from \cref{sec_identif}. We perform clustering as follows:
\begin{enumerate}[label=\rm{(\arabic*)}]
    \item The vertex set $T_1$ is c-isolated in $\cG$ since each parent of the cluster (both observed and unobserved) is also a parent of the vertex $t^* = n_2$, and there are no bidirected paths from $T_1$ to its descendants. Clustering of $T_1$ yields a graph $\cG'$, which is displayed in \Cref{fig_sleepmed_partially_clustered}.
    \item The vertex set $T_2$ is c-coherent with $t^* = s$ in the clustered graph $\cG'$ since $s$ is a parent of each child of the cluster. Clustering of $T_2$ in $\cG'$ yields a graph $\cG''$, which is present in \Cref{fig_sleepmed_clustered}. 
    \item We notice that we can further cluster the c-isolated vertex set $T_3 = \{ t_1, t_2 \}$ in $\cG''$, which yields the graph $\cG^{(3)}$ presented in \Cref{fig_sleepmed_clustered_final}.
\end{enumerate}
According to \cref{thm_identif_c_coherent,thm_identif_c_isolation,coro_iterated_universal}, the iterated clustering operation described above is universally identification invariant. Based on this information we make the following conclusions: 
\begin{itemize}
    \item Causal effects $p(x_m \mid \Do(x_B))$, where $B \cap \{ d, e \} \neq \emptyset$, are nonidentifiable in $\cG$, because the corresponding causal effects are nonidentifiable in $\cG^{(3)}$.
    \item The causal effect $p^{(3)}(x_m \mid \Do(x_p))$ is identifiable in $\cG^{(3)}$ with a functional 
    \begin{equation} \label{eq_identif_in_clustered}
    \begin{aligned}
    p^{(3)}(x_m \mid \Do(x_p)) & = \sum_{x_d, x_e, x_{t_3}} \Bigg[ p^{(3)}(x_{t_3} \mid x_d, x_e, x_p) p^{(3)}(x_e\cond x_d) p^{(3)}(x_d) \\
    & \times \left( \sum_{x_{\tilde p}} p^{(3)} (x_m \mid x_d, x_e, x_{\tilde p}, x_{t_3}) p^{(3)}(x_{\tilde p} \mid x_d, x_e)  \right) \Bigg].
    \end{aligned}
    \end{equation}
    The corresponding causal effect $p(x_m \mid \Do(x_p))$ is identifiable in $\cG$, and applying \cref{thm_identif_functional} iteratively one obtains an identifying functional by first replacing $x_{t_3}$ by the vector $(x_{t_1}, x_{t_2})$, and then replacing $x_{t_1}$ and $x_{t_2}$ by the vectors $x_{T_1} = ( x_a, x_{n_1}, x_{n_2}, x_{n_3 } )$ and $x_{T_2} = (x_c, x_\ell, x_s)$, respectively:
    \begin{align*}
    p(x_m \mid \Do(x_p)) & = \sum_{x_d, x_e, x_{T_1}, x_{T_2}} \Bigg[ p(x_{T_1}, x_{T_2} \mid x_d, x_e, x_p) p(x_e\cond x_d) p(x_d) \\
    & \times \left( \sum_{x_{\tilde p}} p (x_m \mid x_d, x_e, x_{\tilde p}, x_{T_1}, x_{T_2}) p(x_{\tilde p} \mid x_d, x_e)  \right) \Bigg].
    \end{align*}
\end{itemize}


\section{Discussion} \label{sec_discussion}

We studied the role of clustering in causal effect identification from a theoretical viewpoint. We defined coherent clustering operations, which prevent clustering from creating directed cycles or introducing spurious directed paths. Whenever the clustering operation does not create a directed cycle, any causal effect that is identifiable in the clustered graph is identifiable in the original graph as well, and the identifying functional can be obtained from the clustered graph by a simple substitution. This reduces the problem of identification invariance: it suffices to determine when nonidentifiability in the clustered graph implies nonidentifiability in the original graph, or equivalently, when clustering cannot create a hedge that is not already present. We established three conditions based on c-components under which no such hedge is created, namely plainness, c-coherence and c-isolation, each of which renders the clustering of a coherent cluster identification invariant for all causal effects not involving the clustered variables.

These results widen the range of settings in which identification invariance can be used in practice. Coherent clusters generalize transit clusters: they impose no requirement that all members of the cluster that have observed parents outside the cluster must have the same observed parents outside the cluster. An analogous result applies to the children of the cluster. In \Cref{sec_app}, we illustrated the practical applicability of the results.

Each of the three conditions, plainness, c-coherence and c-isolation, is sufficient but not necessary for universal identification invariance, and clustering operations exist that are identification invariant without satisfying any of them (e.g., \Cref{fig_almost_applicable}). The conditions presented here were chosen to be structural properties of the cluster that can be readily verified from the original graph. Characterizing structural conditions that are both necessary and sufficient
remains an open question. The conditions may also be relaxed if invariance is sought only for a smaller family of causal effects (e.g., just one specific causal effect), and in such a setting even the coherence assumption may not be needed, as long as new directed paths are not formed between the variables of the causal effects under study.

We also required that clustering preserves acyclicity. Allowing clustering to create directed cycles could be a potential next step in future research. In this setting, cycles in the clustered graph does not need to correspond to any cycle in the original graph, and identification has been studied through the do-calculus applied directly to the clustered graph \citep{ferreira2025identifying,yvernes2026relaxing}. Imposing conditions for identification invariance appears considerably harder in this case: the characterization of nonidentifiability itself has to be modified, as an ordinary hedge no longer detects the nonidentifiability that a cluster-level cycle induces \citep{ferreira2025identifying}.

Our focus has been on the theoretical properties of clustering operations. These properties provide a basis for future algorithmic solutions that could, for instance, detect a subset of coherent clusters or verify plainness, c-coherence and c-isolation of a cluster. 


\acks{Jani Nykänen and Juha Karvanen were supported by the Research Council of Finland (grant no. 368935 to JK). Otto Tabell was supported by the Finnish Doctoral Program Network in Artificial Intelligence (AI-DOC, Decision VN/3137/2024-OKM-6).} 

\nosectionappendix


\begin{toappendix}

\section{Transit Cluster} \label{apx_transit}

A transit cluster \citep{thk2023} is a special type of a coherent cluster that is uniquely determined by its \emph{receivers} and \emph{emitters}, which we define below.

\begin{defin}[Receiver] \label{def_receiver}
For a set of vertices $T \subset V$ in a causal graph $\cG = (V\cup U, E)$, the set of receivers is the set
\begin{equation*}
    \Rcv_{\cG}(T) = \{r \in T \mid \Pa{\cG}{r} \setminus T^+ \neq \emptyset \},
\end{equation*}
that is, the set of all vertices of $T$ that have parents outside the vertex set $T$ (excluding those unobserved parents whose both children are in $T$).
\end{defin}

We note that receivers also include those vertices that have an unobserved parent whose other child is outside the cluster, that is, vertices that are connected via a bidirected arc to a vertex outside the cluster. 

\begin{defin}[Emitter]
For a set of vertices $T\subset V$ in a causal graph $\cG=(V\cup U, E)$, the set of emitters is the set
\begin{equation*}
    \Emt_{\cG}(T)=\{e \in T \mid \Ch{\cG}{e} \setminus T \neq \emptyset \},
\end{equation*}
that is, the set of all vertices of $T$ that have children outside the vertex set $T$.
\end{defin}

A transit cluster is defined as follows:
\begin{defin}[Transit cluster] \label{def_original_cluster}
A set of vertices $T \subset V$ in a causal graph $\cG = (V \cup U, E)$ is a transit cluster if the following conditions hold:
\begin{enumerate}[label=\rm{(T\arabic*)}]

	\item For all pairs of receivers $r_1, r_2 \in \Rcv_{\cG}(T)$ it holds $\Prnt_{\cG}(r_1) \setminus T^+ = \Prnt_{\cG}(r_2) \setminus T^+$ (i.e., all receivers have the same parents). \label{cond_O_parents} \label{cond_O_first}

	\item For all pairs of emitters $e_1, e_2 \in \Emt_{\cG}(T)$ it holds $\Chl_{\cG}(e_1) \setminus T = \Chl_{\cG}(e_2) \setminus T$ (i.e., all emitters have the same children). \label{cond_O_children} 
	
	\item For all vertices $t_i \in T$ there exists a receiver $r \in \Rcv_{\cG}(T)$ or an emitter $e \in \Emt_{\cG}(T)$ such that $t_i$ and $r$ or $t_i$ and $e$ are connected via an undirected path in the subgraph $\cG[T^=]$, which is obtained from $\cG[T]$ by removing all incoming edges to $\Rcv_{\cG}(T)$ and all outgoing edges from $\Emt_{\cG}(T)$. \label{cond_O_connected}

	\item If $\Emt_{\cG}(T) \neq \emptyset$, then for all $r \in \Rcv_{\cG}(T)$ there exists an $e \in \Emt_{\cG}(T)$ such that $e \in \Dsc_{\cG}(r)$. \label{cond_O_descendants}
	\vspace{-1.0em}
	\item If $\Rcv_{\cG}(T) \neq \emptyset$, then for all $e \in \Emt_{\cG}(T)$ there exists an $r \in \Rcv_{\cG}(T)$ such that $r \in \Anc_{\cG}(e)$. \label{cond_O_ancestors} \label{cond_O_last}
    \vspace{-1.0em}
\end{enumerate}

\end{defin}

\smallskip

\cref{def_original_cluster} deviates slightly from the definition of a transit cluster used by \citet{thk2023} and \citet{ttk2025} since we only allow clustering of observed vertices. However, if $T$ is a transit cluster according to \cref{def_original_cluster}, then the augmented cluster $T^+$ (i.e., the set $T$ with all unobserved vertices whose both children are in $T$) is considered a transit cluster in the aforementioned articles.


\section{Auxiliary Results} \label{sec_appendix_aux}

In coherent clustering "similar" hedges in the original and the clustered graph also share a similar root set, a property which is continuously exploited in the proofs of \cref{sec_identif}. A more precise statement is presented below.

\begin{lemrep} \label{lemma_subgraph_coherence}
Let $A, B \subset V$ be disjoint sets in a causal graph $\cG = (V \cup U, E)$. Suppose a coherent cluster $T \subset V \setminus (A \cup B)$, and let $\cG' = (V' \cup U', E')$ be the graph induced by the clustering of $T$ with a vertex $t \in V'$ representing $T$. Then $\Anc_{\cG[\overline{B}]}(A) \setminus T = \Anc_{\cG'[\overline{B}]}(A) \setminus \left\{ t \right\}$.
\end{lemrep}

\begin{proof}
"$\subseteq$": Let $w_i \in \Anc_{\cG[\overline{B}]}(A) \setminus T$. Then there is a directed path $w_i \rightarrow w_{i + 1} \rightarrow \cdots \rightarrow w_{i + m} \rightarrow a$ in $\cG[\overline{B}]$ such that $w_{i + 1}, \ldots, w_{i + m} \in V$ and $a \in A$. By \cref{lemma_coherence_properties} \ref{prop_coherence_one_maximal_subpath} the vertices in the path that belong to $T$ form a subpath. Since $T \cap B = \emptyset$, we can replace this subpath by the vertex $t$ to obtain a path connecting $v_i$ to $A$ in $\cG'[\overline{B}]$. Thus, $w_i \in \Anc_{\cG'[\overline{B}]}(A)$.

\smallskip

"$\supseteq$:" Let $w_i \in \Anc_{\cG'[\overline{B}]}(A) \setminus \{t\}$. Again there is a directed path $w_i \rightarrow w_{i + 1} \rightarrow \cdots \rightarrow w_{i + m} \rightarrow a$ in $\cG'[\overline{B}]$, where $w_{i + 1}, \ldots, w_{i + m} \in V'$ and $a \in A$. This path can have at most one occurrence of the vertex $t$, because otherwise there would be a directed cycle in $\cG'$. We use \cref{def_coherence} to replace the vertex $t$ in the path by a subpath $t_1 \rightarrow \cdots \rightarrow t_m$ such that $t_1, \ldots, t_m \in T$ in $\cG$. Since $T \cap B = \emptyset$, the substitutive path also exists in $\cG[\overline{B}]$, hence $w_i \in \Anc_{\cG[\overline{B}]}(A)$.
\end{proof}


\section{Explanation of Variables in the Demonstration Example} \label{sec_appendix_table}

\Cref{table_sleepmed} provides explanation of variables appearing in the graphs of \Cref{fig_sleepmed} in \cref{sec_app}.

\begin{table}[!h]
\centering
\begin{tabular}{cl}
\toprule
Vertex & Variable \\
\midrule
$d$ & Demography \\
$e$ & Socioeconomic status \\
$s$ & Noise sensitivity \\
$c$ & Chronic diseases \\
$\ell$ & Lifestyle \\
$p$ & Population density \\
$n_1$ & Neighborhood noise \\
$n_2$ & Rail traffic noise \\
$n_3$ & Road traffic noise \\
$a$ & Noise annoyance \\
$m$ & Sleep medication use \\
\bottomrule
\end{tabular}
\caption{Explanation of variables in \Cref{fig_sleepmed}.} \label{table_sleepmed}
\end{table}

\end{toappendix}

\begin{toappendix}
\clearpage
\bibliography{references}
\label{applastpage} 
\end{toappendix}

\end{document}